\documentclass[sigconf]{aamas} 

\usepackage{tikz}
\usepackage{float}
\usepackage{booktabs}
\usepackage{enumitem}
\usepackage[linesnumbered,lined,boxed,ruled]{algorithm2e}
\SetKwInput{kwInput}{Input}
\SetKwInput{kwOutput}{Output}
\SetKwInput{kwParam}{Parameter}

\usepackage{amsmath}
\usepackage{thm-restate}

\theoremstyle{definition}
\newtheorem{lemma}{Lemma}

\newtheorem{definition}{Definition}
\theoremstyle{definition}

\newtheorem{assumption}{Assumption}

\AfterEndEnvironment{definition}{\noindent\ignorespaces}
\AfterEndEnvironment{lemma}{\noindent\ignorespaces}
\AfterEndEnvironment{theorem}{\noindent\ignorespaces}
\AfterEndEnvironment{assumption}{\noindent\ignorespaces}
\AfterEndEnvironment{observation}{\noindent\ignorespaces}
\AfterEndEnvironment{restatable}{\noindent\ignorespaces}

\allowdisplaybreaks

\let\oldnl\nl
\newcommand{\nonl}{\renewcommand{\nl}{\let\nl\oldnl}}%

\renewcommand{\b}[1]{\boldsymbol{#1}}
\newtheorem{proposition}{Proposition}

\newcommand{\muvec}{\boldsymbol{\mu}}
\newcommand{\lambdavec}{\boldsymbol{\lambda}}
\newcommand{\xvec}{\boldsymbol{x}}
\newcommand{\yvec}{\boldsymbol{y}}
\newcommand{\zvec}{\boldsymbol{z}}

\newcommand{\R}{R}

\newcommand{\uvec}{\boldsymbol{u}}
\newcommand{\vvec}{\boldsymbol{v}}
\newcommand{\gap}{\texttt{gap}}
\newcommand{\Gap}{\texttt{Gap}}

\newcommand{\rvec}{\boldsymbol{r}}
\newcommand{\biggamma}{\mathbf{\Gamma}}
\newcommand{\I}{\mathbf{I}}
\newcommand{\gvec}{\boldsymbol{g}}
\newcommand{\gammavec}{\boldsymbol{\gamma}}
\newcommand{\rhovec}{\boldsymbol{\rho}}
\newcommand{\nuvecpi}{\boldsymbol{\nu}^{\pi}}

\newcommand{\Ihat}{\hat{\mathbf{I}}}

\newcommand{\BLOCKCOMMENT}[1]{}
\usepackage{enumitem}

\usepackage{balance}

\setcopyright{ifaamas}
\acmConference[AAMAS '22]{Proc.\@ of the 21st International Conference
on Autonomous Agents and Multiagent Systems (AAMAS 2022)}{May 9--13, 2022}
{Online}{P.~Faliszewski, V.~Mascardi, C.~Pelachaud,
M.E.~Taylor (eds.)}
\copyrightyear{2022}
\acmYear{2022}
\acmDOI{}
\acmPrice{}
\acmISBN{}

\def\aamas{}
\let\aamas\undefined

\title[Fairness in AMDPs]{Long-Term Resource Allocation Fairness in Average Markov Decision Process (AMDP) Environment}

\author[1]{Ganesh Ghalme$^*$}
\ifdefined\aamas
\thanks{$^{* }$Author names appear alphabetically. Vineet Nair is now at Google Research, India. Code and appendix are available at \url{https://yilunzhou.github.io/fair-average-mdp/}. }
\else
\thanks{$^{* }$Author names appear alphabetically. Vineet Nair is with Google Research, India at the time of publication. Code is available at \url{https://github.com/YilunZhou/fair-average-mdp/}. }
\fi
\affiliation{
  \institution{Technion Israel Institute of Technology}
  \city{Haifa}
  \country{Israel}}
\email{ganeshg@campus.technion.ac.il}

\author[2]{Vineet Nair$^*$}
\affiliation{
  \institution{Technion Israel Institute of Technology}
  \city{Haifa}
  \country{Israel}}
\email{vineetn90@gmail.com}

\author[3]{Vishakha Patil$^*$}
\affiliation{
  \institution{Indian Institute of Science}
  \city{Bangalore}
  \country{India}}
\email{patilv@iisc.ac.in}

\author[4]{Yilun Zhou$^*$}
\affiliation{
  \institution{Massachusetts Institute of Technology}
  \city{Cambridge}
  \state{MA}
  \country{United States}}
\email{yilun@mit.edu}

\begin{abstract}

Fairness has emerged as an important concern in automated decision-making in recent years, especially when these decisions affect human welfare. In this work, we study fairness in temporally extended decision-making settings, specifically those formulated as Markov Decision Processes (MDPs). Our proposed notion of fairness ensures that each state's long-term visitation frequency is at least a specified fraction. This quota-based notion of fairness is natural in many  resource-allocation settings where the dynamics of a single resource being allocated is governed by an MDP and the distribution of the shared resource is captured by its state-visitation frequency. In an \emph{average}-reward MDP (AMDP) setting, we formulate the problem as a bilinear saddle point program and, for a generative model, solve it using a Stochastic Mirror Descent (SMD) based algorithm. The proposed solution guarantees a \emph{simultaneous approximation} on the expected average-reward and fairness requirement. We give sample complexity bounds for the proposed algorithm and validate our theoretical results with experiments on simulated data. 

\end{abstract} 


\newcommand{\BibTeX}{\rm B\kern-.05em{\sc i\kern-.025em b}\kern-.08em\TeX}

\begin{document}

\pagestyle{fancy}
\fancyhead{}
\maketitle 

\section{Introduction}
\label{intro}
Algorithms are increasingly used to make important decisions that impact human lives. While algorithmic decision-making frameworks offer increased efficiency, speed and scalability, their potential bias and unfairness have led to several concerns. For instance, studies have shown that the traditional algorithms may be unfair towards certain demographics of the population in recidivism prediction \citep{dressel2018accuracy}, loan and credit lending \citep{berkovec1996mortgage}, online advertising \citep{ali2019discrimination}, and recommendation systems \citep{yao2017beyond}. These concerns have led to a surge in research efforts aimed at ensuring fairness in algorithmic decision-making frameworks \citep{Barocas2018FairnessAM}. A large body of work in fair machine learning has focused on either one-shot settings such as classification \citep{dwork2012fairness, kleinberg2017inherent} or sequential but static settings such as multi-armed bandits where reward distributions are stationary \citep{celis2019controlling,zhang2020Survey,patil2020achieving}. However, in many real-world applications, the algorithm's decisions may have long-term impact to the states and rewards of the system. The study of fairness in such temporally-extended decision-making settings, often modeled using the reinforcement learning framework \citep{sutton2018reinforcement}, is still in its infancy. 

We introduce the problem of state-visitation fairness in Markov Decision Processes (MDPs). Informally, our fairness notion requires that each state of the MDP be visited with a pre-specified minimum frequency. In particular, a policy $\pi$ with stationary state distribution $\nuvecpi$, is called \emph{fair} if $\nu^\pi_s \geq \rho_s$ for every state $s$, where $\rho_s \in [0,1)$ specifies the fairness constraints and is given as input to the algorithm. Motivated by recent works on fairness in multi-armed bandits (MAB) \citep{li2019combinatorial, patil2020achieving} that enforce minimum frequency on the selection of each arm, our \emph{quota-based} notion of fairness is natural in dynamic resource allocation settings where the transition dynamics of the resource being allocated is governed by an MDP and an algorithm is required to \emph{equitably} divide the shared resource. Although MABs can be considered a special case of MDPs with one action per arm on a shared dummy state, distinctively different techniques are needed to satisfy the analogously defined constraints on state visitation frequency for general MDPs due to long term implications of the taken actions. 

As a concrete example, consider the task of scheduling the jobs of $C$ clients to run on one shared server, where $C$ is assumed to be fixed and known \textit{a priori}. At every time step, each client submits some number of jobs, and the server decides how many jobs to run for each client. The maximum number of jobs that the server can run at every step is $N$. From the server's perspective, this task can be modeled as an MDP \citep{white1973}: a state is a vector of the form $(n_1, ..., n_C)$ representing the number of remaining jobs for the clients. At this state, the server can take an action $a = (m_1, ..., m_C)$, where $m_c \leq n_c$ is the number of jobs that the server runs for client $i$ at the current time step, subject to the constraint that $\sum_{c=1}^C m_c \leq N$. The transition models the job execution and (potentially stochastic) job arrival. The server gets rewards depending on the number of finished jobs. In the simple case where client $c$ pays $r_c$ for each completed job, the total reward at each step is $\sum_{c=1}^C m_cr_c$, which obviously leads to a strict prioritization toward the highest paying client for a reward maximizing agent. However, in many cases we want to enforce some quality of service (QoS) to every client. Such a guarantee can be modeled with our notion of fairness by requiring minimum frequencies on states with low values of $n_c$ (i.e. remaining jobs) for each low-paying client $c$. 

Similar fairness requirements are present in other resource allocation settings such as taxi dispatching and postal service, and they can also be framed as state-visitation guarantees. In addition, since the service is ``long-running'', the average reward captures the long term profitability of the service provided better than the discounted reward, and motivates us to consider the case of average-reward MDP, which we refer to as AMDP. 

Last, our notion of fairness can also generalize demographic parity \citep{calders2009} into a temporally extended setting. Consider that a company wishes to ensure racial diversity in their employee base, for an extended period of time rather than for a specific hiring decision. We can represent the overall employee profile as a state in an MDP, from which we can compute diversity statistics. The company takes various human-resource decisions on a day-to-day basis that may cause the employee demographics to change. While the diversity objective may not be feasible to achieve at all time (e.g. due to random resignation decisions made by employees), the company still wants to or is required to maintain the diversity with high frequency, which can be naturally encoded as minimum frequency constraints on states that satisfy the diversity objective. 

\subsection{Our Contributions}
In this paper, we make two contributions, a new notion of fairness as constraints on a MDP and an algorithm to solve for the optimal policy under the constraints. On the former, we introduce the fairness notion of \emph{minimum resource allocation guarantee} in the MDP setting. This fairness notion is practically significant, capturing many real world applications and complementing existing notions of fairness in MDP such as approximate action-fairness guarantees \citep{jabbari2017fairness} and demographic parity \citep{wen2019fairness}.

On the latter, our work contributes to the long line of literature on constrained MDPs \citep{altman1999constrained}, which has been mostly used to ensure the safety of exploration \citep{achiam2017constrained}. These formulations usually encode safety as upper bounds on state-dependent cost functions, while our desired fairness constraints are lower bounds on state-visitation frequencies. With sample access to the transition function (i.e. being able to sample the next state given the current state and action), we formulate the problem as a bilinear saddle point problem and present an algorithm (Algorithm \ref{algorithm: fair state visitation} that uses the classical stochastic mirror descent (SMD) framework to simultaneously satisfy the fairness constraint and achieve reward maximization asymptotically (Theorems \ref{thm:SMDGapTheorem} and \ref{thm: optimality and fairness}). Its running time depends on the required approximation threshold $\varepsilon$, the fairness constraint parameter $d_{\rhovec}$, mixing time $t_{mix}$, number of states $n$ and number of actions $m$ as $O\left(nm \varepsilon^{-2} (1+d_{\rhovec})^2 t_{mix}^2 \log(nm)\right)$. 

Recently, \citet{jin2020efficiently} proposed the first algorithm with sample complexity bounds to compute an approximately-optimal policy for unconstrained AMDPs. Our algorithm is similar in spirit to this unconstrained algorithm in \citep{jin2020efficiently} but requires novel analytical techniques to prove the simultaneous guarantee on fairness and reward. The main technical novelty in our work is how the primal variables are bounded. Since the fairness constraints introduce new primal variables in the linear program (Fair-LP (P) in Section \ref{subsec: linear program approach}), the analysis of \citep{jin2020efficiently} does not lend itself to a straightforward extension. Instead, Lemma \ref{lemma: constraining the domain of lambda} non-trivially use structure of the constraint matrix to bound the domain of primal variables in our algorithm. Another technical novelty is that we restrict the dual space by incorporating the fairness constraint explicitly in the domain, which simplifies the objective function (Eq. \ref{equation: min max with h} to Eq. \ref{equation: the final min maz optimization}) and helps achieve the objective. In comparison, \citet{jin2020efficiently} only compute a feasible policy with no regard to optimality. To the best of our knowledge, our work is the first to achieve the simultaneous guarantee with sample complexity bounds.

\section{Related Work}
Recently, there has been growing interest in studying fairness in sequential decision-making. For example, \citet{CREAGER2020} propose causal modeling of dynamical systems to address fairness, \citet{zhang2020fair} study how algorithmic decisions impact the evolution of feature space of the underlying population modeled as an MDP, and \citet{damour2020} study the impact of feedback dynamics on long-term fairness via simulations.

The study of fairness in reinforcement learning (RL) was initiated by \citet{jabbari2017fairness}, who extend the \emph{meritocratic fairness} notion defined by \citet{joseph2016fairness} in the MAB setting to the MDP setting. Under this notion, a policy is fair if, with high probability, an action with a lower long-term reward is not favored over an action with a higher long-term reward. This notion of fairness can be classified as procedural fairness and is different from our outcome-based notion of fairness where the fairness guarantee can be quantified in terms of the state-visitation frequency. \citet{DOROUDI2018} study the problem of off-policy policy selection in RL under similar fairness constraints. In terms of the fairness constraints, the work closest to ours is that of \citet{wen2019fairness}, which models the agents as states of an MDP and study demographic parity with respect to reward to the agents. However, their work is different from ours in three key aspects: 1) they study discounted-reward MDPs in contrast to our average-reward MDPs, 2) they focus primarily on the setting where the model is known, whereas our main contribution is for the generative model, and 3) they model the constraints in terms of a reward to the agents in contrast to our fairness constraints which capture the absolute long-term state-visitation frequency.

The unconstrained AMDP problem has been extensively studied in the literature \citep{Mahadevan96,KearnsS98,brafman2002r}. If the model is known and the MDP is unichain, then \citet{altman1999constrained} and \citet{puterman2014markov} showed that an optimal policy can be computed by solving a linear program. With a generative model (i.e. a simulator that can sample from the transition function and compute teh reward function \citep{kearns1999approximate}), the SMD approach \citep{nemirovski2009robust,carmon2019variance} was recently used by \citet{wang2017} and \citet{jin2020efficiently} to compute an approximately optimal policy. Furthermore, \citet{jin2020efficiently} proposed an algorithm to compute a feasible, but not necessarily optimal, policy for constrained AMDPs. 

We formulate MDPs with fairness guarantees as constrained MDPs (CMDPs), \citep{altman1999constrained}, which have been studied extensively in the safety setting \citep{achiam2017constrained}. The policy search seeks to maximize the reward while ensuring certain upper bound frequency on some (high-risk or error) states \citep{geibel2005risk, tamar2012policy}. By contrast, our constraints are defined as lower bounds on states, for which existing techniques could not be adapted in a straightforward manner. Last, another line of literature studies the problem where the objective itself is to achieve a specified state-visitation frequency (or some function of it) in the absence of reward signals \citep{hazan2019provably, lee2019efficient}.

In the rest of the paper, Section \ref{sec: fair AMDP model} introduces the Fair-AMDP problem. Section \ref{sec: constrained domain and SMD} formulates the solution as a linear program and provide necessary background on the SMD framework. Section \ref{sec:algorithm} presents the concrete algorithm implementation and Section \ref{sec: theoretical results} presents the theoretical analysis. Section \ref{sec:experiments} presents experimental results to validate the proposed algorithm.
Finally, Section \ref{sec:discussion} presents a discussion of the work and some future directions.


\section{Fair-AMDP Model }\label{sec: fair AMDP model}
A discrete Markov Decision Process (MDP) is a sequential decision-making framework denoted by the tuple $\langle \mathcal{S}, \mathcal{A}, \mathbf{\Gamma}, \rvec, \uvec \rangle$. At each step $t \geq 1$, $s_t \in \mathcal{S}$ denotes the state of the MDP at time $t$. A decision-maker takes an action $a_t \in \mathcal A$, receives a finite reward $r_t = r_{s_t, a_t}$, and the MDP transitions to a state $s_{t+1}$ according to the transition probability function $\mathbf{\Gamma}: \mathcal{S} \times \mathcal{A} \rightarrow \Delta^{|\mathcal{S}|}$ where $\Delta^{|\mathcal{S}|}$ is the simplex set of distributions over states. Without loss of generality, we assume that the rewards are non-negative and depend only on $s_t$ and $a_t$ and not on $s_{t+1}$. The initial state $s_1$ is sampled from the initial-state distribution $\uvec \in \Delta^{|\mathcal{S}|}$. 

Let $\pi$ be a stochastic policy with $\pi_{s,a}$ denoting the probability with which action $a$ is taken in state $s$. Each policy $\pi$ induces a stationary distribution over the state space denoted by $\nuvecpi \in \Delta^{n}$. One of the popular optimization problems under an MDP framework \cite{altman1999constrained, puterman2014markov} is to find a policy that maximizes the long-term expected average-reward given by
\begin{equation}
\lim_{T \rightarrow \infty }\frac{1}{T} \cdot \sum_{t=1}^T \mathbb E_{\pi}\left[r_t\right]. \label{rl-obj}
\end{equation}

Throughout the paper, we consider a finite MDP with $\mathcal{S} = \{1, 2, \ldots, n\}$ and $\mathcal{A} = \{1, 2, \ldots, m\}$. Let $\ell = nm$ denote the total number of state-action pairs. For clarity, we often use $(s, a)$ to index the $((s-1)m + a)$-th entry of vectors in $\mathbb{R}^{\ell}$. Thus, we can equivalently represent the transition function $\Gamma$ as a matrix of dimension $\ell \times n$, where $\Gamma_{(s, a), s'}$ is the probability of going to $s'$ when taking $a$ at $s$. Similarly, we can represent the reward function $\b{r}$ as a vector in $\mathbb R^l$, where $r_{(s, a)}$ gives the reward of taking $a$ at $s$. For notational convenience, we define a matrix form $\Pi\in \mathbb R^{l\times n}$ for the policy $\pi$ where $\Pi_{(s, a), s}=\pi_{s, a}$ and $\Pi_{(s, a), s'}=0$ for all $s\neq s'$. Thus, $\mathbf{\Gamma}^{\pi} := \Pi^T\mathbf{\Gamma}$ is the transition matrix of the Markov chain induced by $\pi$. In this work, we restrict ourselves to ergodic MDPs, defined below. 

\begin{definition}{(Ergodicity)}
A Markov decision process is ergodic if Markov chain on the state induced by any policy is ergodic. A Markov chain is ergodic if there exists a positive integer $T_0$, such that for all pairs of states $s_i, s_j$, if the chain is started at $s_i$, the probability of being in state $s_j$ is non-zero for all time after $T_0$. 
\end{definition}

Intuitively, the state transition in an ergodic MDP mixes across all states without showing any periodic oscillations. Next, we define its mixing time as follows. 

\begin{definition}{(Mixing Time)}
The mixing time of a given MDP $\langle \mathcal{S},\mathcal{A}, \mathbf{\Gamma}, \rvec, \uvec \rangle$ is given by $t_{mix} = \max_{\pi} t_{\pi}$ where, 
\begin{align*}
t_{\pi} = \arg \min_{t \geq 1} 
\left[ \max_{\uvec} \left[ ||(\mathbf{\Gamma}^{\pi^{T}})^{t} \uvec - \nuvecpi||_{1} \leq 1/2 \right] \right].
\end{align*}
\end{definition}

The mixing time $t_{\text{mix}}$ of an ergodic MDP captures how fast the Markov chain induced by any policy converges to its corresponding stationary distribution.

\begin{assumption}
\label{assumption: ergodic}
The MDP instance $\langle \mathcal{S}, \mathcal{A}, \mathbf{\Gamma}, \rvec, \uvec \rangle$ is ergodic. 
\end{assumption}

The ergodicity condition enables us to formulate the problem of finding an optimal policy for an AMDP problem as a linear program (Section \ref{subsec: linear program approach}). We consider a constrained version of it where the constraints are in terms of the minimum state-visitation frequency. In particular, we study the Fair-AMDP problem with the following notion of fairness. 

\begin{definition}
Let $\boldsymbol{\rho} \in [0,1]^{n}$ such that $\sum_s \rho_s \leq 1$. Then, a policy $\pi$ is called $\rhovec$-fair if $ \nu_{s}^{\pi} \geq \rho_{s} $ for all $s \in [n]$. 
\end{definition}

A Fair-AMDP instance is denoted as $\langle \mathcal{S}, \mathcal{A}, \mathbf{\Gamma}, \rvec, \uvec, \rhovec \rangle$. We note that a Fair-AMDP instance may not even have a feasible policy. For example, consider an AMDP instance with $n=2,m=1$ and $\mathbf{\Gamma} = ( 1 - \alpha, \alpha ; 1 - \alpha, \alpha)$. Here, the stationary distribution for the unique policy $\pi $ is $\nuvecpi = (1 - \alpha, \alpha) $. There does not exist a $\rhovec$-fair policy for any $ \rhovec$ with $\rho_{2} > \alpha$. We restrict our attention to $\boldsymbol{\rho}$ such that $\rho_s < 1/n$ for all $s \in \mathcal{S}$ unless otherwise specified.In this paper, we assume that the problem is feasible. 
\begin{assumption}
\label{asm:fairAction}
There exists a $\b{\rho}$-fair policy. 
\end{assumption}

Ensuring feasibility and identifying infeasibility can be done with one technical addition to the generative model. We introduce a fair action $a^*$ available at each state with a reward that is strictly lower than the one on any state-action pair. One choice of $a^*$ is such that $\mathbf{\Gamma}((s, a^{\star}), s' ) = 1/n$. In other words, taking this action ``resets'' the agent to a state selected uniformly at random. We make the following additional remarks regarding the above assumption.

\begin{enumerate}[leftmargin=*]
    \item Given a Fair-AMDP with action $a^*$, a policy $\pi$ that chooses action $a^*$ in all the states $s$ has $\nuvecpi = (1/n \ldots 1/n)$ and hence $\nuvecpi > \boldsymbol{\rho}$ implying that the Fair-AMDP is strictly feasible. Further, if the Fair-AMDP instance is guaranteed to have a feasible solution without using the fair-action then there is an optimal policy that has zero probability of choosing action $a^*$ at any state $s$. 
    \item For a particular $\rhovec$, one may relax the assumption to the following: the fair-action $a^*$ is such that $\mathbf{\Gamma}((s,a^*), s') >\max_s \rho_s$ for all $s, s'\in [n]$. Such an assumption is seemingly necessary to compute a fair-algorithm in the generative model where there is only stochastic access to the transition probability matrix via state-action queries. 
\end{enumerate}

\section{Solution Approach}
\label{sec: constrained domain and SMD}

As mentioned before, we assume access to a \emph{generative model}, which can be used to sample the next state $s'$ given the current state $s$ and action $a$ according to $\mathbf\Gamma$ and compute the reward $\mathbf r_{s, a}$. 

\subsection{Linear Program for Fair-AMDP}
\label{subsec: linear program approach}
We first recall the linear programming solution to solve the (uncontrained) AMDP when the transition and reward functions are known. The formulation is based on the Bellman equation for optimal policy, and is derived in detail in textbooks \citep[e.g.][]{puterman2014markov}. Let $\Ihat \in \mathbb{R}^{\ell \times n}$ be such that $\Ihat((s,a), s') = 1$ if $s = s'$ and $0$ otherwise. 

\vspace{1em}

{ \centering
\begin{tabular}{c|c}\toprule
 \textsc{UC-LP (D)} & \textsc{UC-LP (P)} \\
 $\displaystyle \max_{\b{x} \in \Delta^{\ell}} \b{x}^{T}\rvec$ & 
 $\displaystyle \min_{\b{\lambda} > \b{0}, \beta} \beta$\\
 subject to $(\Ihat - \mathbf{\Gamma})^T \b{x} = \b{0}$ & subject to $(\Ihat - \mathbf{\Gamma})\b{\lambda} + \b{r} \leq \beta^T \b{1}$\\\bottomrule
\end{tabular}

}

\vspace{1em}

We focus our attention on \textsc{UC-LP (D)}. From the optimal solution $\b x$, we can derive the policy $\pi$ as 
\begin{align*}
    \pi_{s,a} = \frac{x_{s,a}}{\sum_{a'=1}^m x_{s,a'}}
\end{align*}
It is easy to verify the following: 
\begin{enumerate}[leftmargin=2em]
    \item For $\nuvecpi = (\sum_{a}x_{1,a}, \ldots, \sum_{a}x_{n,a})$, we have $\b{x} = \Pi\b{\nu}^{\pi}$, 
    \item $\b{\nu}^{\pi}$ is the stationary state distribution corresponding to the $\pi$, as $(\Ihat - \mathbf{\Gamma})^T\b{x} = (\mathbf{I} - \mathbf{\Gamma}^{\pi})^T \b{\nu}^{\pi} = \b{0}$, and 
    \item $\Ihat\b{x} = \b{\nu}^{\pi}$. 
\end{enumerate}
In particular, from (3) we conclude that $x_{s,a}$ is the average state-action visitation probability of the state-action pair $(s,a)$. Note that Assumption \ref{assumption: ergodic} ensures that $\nuvecpi > \b{0}$, for every such $\pi$.

It follows that the desired fairness guarantee can be achieved by ensuring that $\b{x}$ satisfies $\sum_{a} x_{s,a} \geq \rho_s$ $\forall s$ such that $\rho_s > 0$. We assume, without loss of generality, that $\rho_s >0$ $\forall s$. We state the \textsc{Fair-LP} primal/dual below, where $ \mathbf{C} \in \mathbb{R}^{n\times \ell}$ is such that the $(s',(s,a))$-entry is $1/\rho_s$ if $s=s'$ and $0$ otherwise. Assumption \ref{asm:fairAction} guarantees their feasibility.

\vspace{1em}
{\centering 
\resizebox{\columnwidth}{!}{
\begin{tabular}{c|c}\toprule
 \textsc{Fair-LP (D)} & \textsc{Fair-LP (P)} \\
 $\displaystyle \max_{\b{x} \in \Delta^{\ell}}\b{x}^{T}\rvec$ & $\displaystyle \min_{\b{\lambda} > \b{0}, \muvec > \b{0}, \beta}\beta - \b{\mu}^T\b{1}$ \\
 subject to $(\hat{\mathbf{I}} - \mathbf{\Gamma})^T \b{x} = \b{0}, \mathbf{C}\b{x}\geq \b{1}$ & subject to $(\hat{\mathbf{I}} - \mathbf{\Gamma})\b{\lambda} + \b{r} + \mathbf{C}^T\b{\mu} \leq \beta^T \b{1}$ \\\bottomrule
\end{tabular}
}

}
\vspace{1em}

\begin{proposition}\label{thm:feasibleLP}
Let $\langle \mathcal{S}, \mathcal{A}, \mathbf{\Gamma} , \rvec,\uvec, \rhovec \rangle$ be a Fair-AMDP instance and $\pi$ be a policy with $\nuvecpi$ as the induced stationary distribution. Further, let $\xvec = \Pi \nuvecpi$. Then, $\pi$ is a $\rhovec$-fair policy if and only if $\xvec$ is a feasible solution of $\textsc{Fair-LP}$.
\end{proposition}
Proposition \ref{thm:feasibleLP} establishes that if $\mathbf{\Gamma}$ and $\b{r}$ are known, we can compute an optimal $\rhovec$-fair policy by solving \textsc{Fair-LP}. However, when the model parameters are unknown or partially known, even the problem of verifying the feasibility is difficult. For instance, if $\alpha$ is not known in our example from Section \ref{sec: fair AMDP model}, then without the fair-action, it is impossible to determine whether a policy is $\rhovec$-fair.

\noindent Let $(\b{\lambda}^*, \b{\mu}^*, \beta^*)$ be an (optimal) solution of \textsc{Fair-LP (P)}. First, we note that $\b{\lambda}^*$ is not unique, as for every $c\in \mathbb{R}$, $(\b{\lambda}^* + c\cdot \b{1}, \b{\mu}^*, \beta^*)$ is also an optimal solution. Hence, we may assume without loss of generality that $\b{\lambda}^*$ is orthogonal to the stationary distribution $\b{\nu}^{\pi^{*}}$ of the optimal policy $\pi^*$ of a given Fair-AMDP instance. We conclude this section by stating two important lemmas regarding the nature of $ \b{\mu}^*$ and $\b{\lambda}^*$. The presence of a strictly feasible solution is used in Claim \ref{lemma: constraining the domain of q}, which is in turn used to prove Lemma \ref{lemma: constraining the domain of lambda}. 
\begin{restatable}{claim}{constrainingq}
\label{lemma: constraining the domain of q}
Let $\b{\lambda}^{*}, \b{\mu}^*, \beta^*$ be the solution to the \textsc{Fair-LP (P)}. Then $\mu_s^* \leq (n \rho_s)/(1 - n\rho_s)$ for all $s\in [n]$.
\end{restatable}

\begin{proof}
Since $\b{\lambda}^*$, $\b{\mu}^*$, and $\beta^*$ satisfies the \textsc{Fair-LP} primal, we have
$$(\Ihat - \mathbf{\Gamma}) \b{\lambda}^* + \rvec + \mathbf{C}^{T}\muvec^* \leq \beta^* \cdot \b{1}~.$$
Consider the policy $\pi^{f}$ which only pulls the fair-action at every point. Then $\nu^f = (\frac{1}{n}\ldots \frac{1}{n})$ is the stationary distribution corresponding to $\pi^f$. Since the reward for pulling the control action in any state is zero, $(\Pi^f \b{\nu}^f)^T\b{r} = 0$. Multiplying the above equation by $(\Pi^f \b{\nu}^f)^T$, we have 
$$\sum_{s}\frac{1}{n\rho_s}\mu_i^{*} \leq \beta^* $$
Further, using strong duality we have $\b{r}^T\b{x}^{*} +\sum_{s}\mu_s^* = \beta^*$. Hence,
$$\sum_{s}\frac{1}{n\rho_s}\mu_s^{*} \leq \b{r}^T\b{x}^{*} +\sum_{s}\mu_s^* ~.$$
Since $\b{r} \in [0,1]^{\ell}$, and $\b{x}\in \Delta^\ell$,
$$\sum_{s}\frac{1 - n\rho_s}{n\rho_s}\cdot\mu_s^{*} \leq \b{r}^T\b{x}^{*} \leq 1 ~.$$
This implies for all $s$, $\mu_s^* \leq (n\rho_s)/(1 - n\rho_s)$ assuming $\rho_s< 1/n$.
\end{proof}

\noindent \textbf{Remark}: If we had prior knowledge that there exists a strictly feasible policy $\pi$, then it can be shown that $\mu_s^*(\frac{\nu_s}{\rho_s} - 1) \leq 1 - (\Pi\b{\nu})^T\b{r}$, where $(\Pi\b{\nu})^T\b{r}$ is the reward of the strictly feasible policy $\pi$ and is at least $0$.

We first state the following useful lemma (Lemma 5 by \citet{jin2020efficiently}) here without a proof. This lemma is used in the proof of Lemma \ref{lemma: constraining the domain of lambda} and Theorem \ref{thm: optimality and fairness}.

\begin{lemma}
\label{lemma: from JIN SIDFORD}
Give a mixing AMDP with mixing time $t_{\text{mix}}$, a policy $\pi$, and its transition probability matrix $\mathbf{\Gamma}^{\pi} \in \mathbb{R}^{n\times n}$ and stationary distribution $\b{\nu}^{\pi}$, the following holds:
$$||(\mathbf{I} - \mathbf{\Gamma}^{\pi} + \b{1}\b{\nu}^T)^{-1}||_{\infty} \leq 2t_{\text{mix}}~.$$
\end{lemma}

Our proposed algorithm's convergence time depends on the mixing time of the MDP via the following lemma.

\begin{restatable}{lemma}{constraininglambda}\label{lemma: constraining the domain of lambda}
Let $\b{\lambda}^{*}, \b{\mu}^*, \beta^*$ be the solution to the \textsc{Fair-LP (P)}. Then $||\b{\lambda}^*||_{\infty} \leq M := 2t_{\text{mix}}(1+d_{\b{\rho}})$, where $d_{\b{\rho}} = \max_s \frac{n}{1-\rho_s n}$.
\end{restatable}
\begin{proof}
Since $\b{\lambda}^*, \b{\mu}^*$, and $\beta^*$ is the solution to the \textsc{Fair-LP} primal,
\begin{equation}\label{equation: lemma for constraining the lambda 1}
(\Ihat - \mathbf{\Gamma})\b{\lambda}^* + \mathbf{C}^T\b{\mu}^* + \b{r} \leq \beta^* \b{1} ~. 
\end{equation}
Let $\pi^*$ be the optimal policy corresponding to $\b{x}^{*}$, and let $\Pi^* \in \mathbb{R}^{\ell\times n}$ and $\b{\nu}^*$ be its corresponding matrix and stationary distribution respectively. Note that $\mathbf{\Gamma}^{*} = (\Pi^*)^T\cdot \mathbf{\Gamma}$ is the probability transition matrix corresponding to the Markov chain induced by $\pi^*$. Multiplying Equation \ref{equation: lemma for constraining the lambda 1} by $(\Pi^*)^T$ from the left, and using the KKT condition and that $\nu^*_s >0$ for all $s$ (as the MDP is ergodic; Assumption \ref{assumption: ergodic}) we have 
\begin{equation}\label{equation: lemma for constraining the lambda 2}
(\mathbf{I} - \mathbf{\Gamma}^*)\b{\lambda}^* + \mathbf{D}_{\b{\rho}}\muvec^* + (\Pi^*)^T\b{r} = \beta^* \cdot \b{1}, 
\end{equation}
where $\mathbf{D}_{\rhovec}$ is the $n\times n$ diagonal matrix with its $s$-th entry being $\frac{1}{\rho_s}$. It is easy to see that $(\Pi^*)^T\mathbf{C}^T = \mathbf{D}_{\b{\rho}}$, and $(\Pi^*)^T\b{1} = \b{1}$. Denote $(\Pi^*)^T\b{r}$ as $\rvec^*$. Since $\langle\b{\lambda}^*,\b{\nu}^*\rangle = 0$, Equation \ref{equation: lemma for constraining the lambda 2} can be rewritten as follows 
$$(\mathbf{I} - \mathbf{\Gamma}^{*} + \b{1}(\b{\nu}^*)^T)\b{\lambda}^* + \mathbf{D}_{\b{\rho}}\muvec^* + \rvec^* = \beta^* \cdot \b{1}~.$$
This implies
\begin{align}
 \b{\lambda}^* &= ~(\mathbf{I} - \mathbf{\Gamma}^{*} + \b{1}(\b{\nu}^*)^T)^{-1}(\beta^* \cdot \b{1} -\mathbf{D}_{\b{\rho}}\muvec^* - \rvec^*) \nonumber\\
 ||\b{\lambda}^*||_{\infty} &\leq ~ ||(\mathbf{I} - \mathbf{\Gamma}^{*} + \b{1}(\b{\nu}^*)^T)^{-1} ||_{\infty} ||\beta^* \cdot \b{1} -\mathbf{D}_{\b{\rho}}\muvec^* - \rvec^* ||_{\infty} \nonumber \\
 &\leq ~ 2t_{\text{mix}}||\beta^* \cdot \b{1} -\mathbf{D}_{\b{\rho}}\muvec^* - \rvec^* ||_{\infty} \nonumber\\
 &\leq~ 2t_{\text{mix}}(1+ \max_s \frac{n}{1-\rho_s n}) \label{equation: bounding infinity norm of lambda*}
\end{align}
The first inequality in Equation \ref{equation: bounding infinity norm of lambda*} follows from Lemma \ref{lemma: from JIN SIDFORD}, and the second follows by showing that $||\beta^* \cdot \b{1} -\mathbf{D}_{\b{\rho}}\muvec^* - \Pi^T\b{r} ||_{\infty} \leq (1+ \max_s \frac{n}{1-\rho_s n})$, which we argue next. By strong duality we have $\beta^* - \sum_s \mu_s^* = \b{r}^T\b{x}^*$. Since $\b{r}\in [0,1]^{\ell}$ and $\b{x}^*\in \Delta^\ell$, $0 \leq \beta^* \leq 1+ \sum_s \mu_s^*$. Hence from Claim \ref{lemma: constraining the domain of q}, we have for any $s$
$$-1 - \sum_s \frac{n\rho_s}{1-n\rho_s} \leq \beta^* - \frac{\mu^*_s}{\rho_s} - r^*_s \leq 1 + \max_s \frac{n}{1-n\rho_s} \ ~. $$
Now as $\sum_s \rho_s <1 $, \ $1 + \sum_s \frac{n\rho_s}{1-n\rho_s} \leq 1 + \max_s \frac{n}{1-n\rho_s})$. Hence, $||\beta^* \cdot \b{1} -\mathbf{D}_{\b{\rho}}\b{\mu}^* - \b{r}^* ||_{\infty} \leq (1+ d_{\b{\rho}})$.
\end{proof}
The bound on $||\b{\lambda}^*||_{\infty}$ crucially restricts the search space of the primal variable $\b{\lambda}$ in the stochastic mirror descent (SMD) approach in Section \ref{subsection: Min-Max Problem} and \ref{subsection: SMD framework} and enables us to give convergence guarantee for the proposed algorithm. 
\subsection{Minimax Formulation}
\label{subsection: Min-Max Problem}

We formulate Fair-AMDP as a bilinear saddle point problem using the Lagrangian of \textsc{Fair-LP}, 
\begin{equation*}
 h(\b{x}, \b{\lambda}, \b{\mu}, \beta ) = \rvec^T\b{x} + \lambdavec^T(\Ihat - \mathbf{\Gamma})^T\b{x} + \b{\mu}^{T}( \mathbf{C} \b{x} - \b{1} ) + \beta (1 - \b{1}^{T} \b{x}).
\end{equation*}
Let $\b{x}^*, \b{\lambda}^{*}, \b{\mu}^*, \beta^*$ be a solution to 
\begin{equation} \label{equation: min max with h}
 \min_{\b{\lambda} \in \mathbb{R}^n, \b{\mu}\geq \b{0}, \beta}~~\max_{\b{x} \geq\b{0}} ~~h(\b{x}, \b{\lambda}, \b{\mu}, \beta )~.
\end{equation}
From Assumption \ref{asm:fairAction}, we have that \textsc{Fair-LP} problem has a feasible solution. Hence, $\b{x}^*$ is the solution to \textsc{Fair-LP (D)}, and $\b{\lambda}^{*}, \b{\mu}^*$ and $\beta^*$ is the solution to \textsc{Fair-LP (P)}. Moreover, from the KKT optimality conditions, $(\b{\mu}^*)^{T}( \mathbf{C} \b{x}^* - \b{1} ) = 0$.

Let $\Delta_{\b{\rho}}^{\ell} := \{\b{x}\in \Delta^{\ell} \mid \sum_{a}x_{s,a}\geq \rho_s ~~\text{for all } s\}$, and $\mathbb{B}_{2M}^{n} := \{\b{\lambda} ~\mid ~ ||\b{\lambda}||_{\infty} \leq 2M\}$, where $M$ is as defined in Lemma \ref{lemma: constraining the domain of lambda}. Then, for every $\b{x}^* \in \Delta_{\b{\rho}}^{\ell}$, note the following: a) $\beta (1 - \b{1}^{T} \b{x}) = 0$, b) $\b{0} \in \arg\min_{\b{\mu}\geq \b{0}}\b{\mu}^{T}(\mathbf{C} \b{x} - \b{1})$. Hence, using Lemma \ref{lemma: constraining the domain of lambda} it follows that $(\b{x}^*, \b{\lambda}^{*}, \b{\mu}^*, \beta^*)$ is a solution to the optimization problem in Equation \ref{equation: min max with h} if and only if $(\b{x}^*, \b{\lambda}^{*})$ is the solution to the following problem 
\begin{equation}\label{equation: the final min maz optimization}
 \min_{\b{\lambda} \in \mathbb{B}_{2M}^n} \ \ \max_{\b{x}\in \Delta_{\b{\rho}}^{\ell}} f(\b{x}, \b{\lambda}) = \rvec^T\b{x} + \b{\lambda}^{T}(\Ihat - \mathbf{\Gamma})^T\b{x}. 
\end{equation}
Furthermore, we have 
\begin{align*}
h(\b{x}^*, \b{\lambda}^*, \b{\mu}^{*}, \beta^*) = f(\b{x}^*, \b{\lambda}^*) = \b{r}^T\b{x}^*.
\end{align*}
Though $|\lambdavec^*|_{\infty} \leq M$, increasing the domain of $\b{\lambda}$ from $M$ to $2M$ helps in the proof of Theorem \ref{thm: optimality and fairness} which provides simultaneous approximation guarantees on the optimality and fairness of the policy that is constructed by Algorithm \ref{algorithm: fair state visitation}.

Next, we define the $\Gap$ function, which quantifies the closeness of a given feasible solution $(\xvec, \lambdavec )$ to the optimal solution.
\begin{definition}
The gap function $ \Gap: \Delta_{\rhovec}^{\ell} \times \mathbb{R}^{n} \rightarrow \mathbb{R}_{+}$ is defined as 
\begin{equation}
 \Gap(\xvec, \lambdavec ) = \max_{\xvec' \in \Delta_{\rhovec}^{\ell}} f(\xvec', \lambdavec) - \min_{\lambdavec' \in \mathbb{B}_{2M}^{n} } f(\xvec, \lambdavec').
\end{equation}
\end{definition}
It is easy to see that the $\Gap$ function is non-negative and $\Gap(\xvec^{\star}, \allowbreak \lambdavec^{\star}) = 0$. We say that $(\xvec, \lambdavec)$ is an $\varepsilon$-approximate solution to Equation \ref{equation: the final min maz optimization} if $\Gap(\xvec,\lambdavec) \leq \varepsilon$.

\subsection{Stochastic Mirror Descent (SMD)}
\label{subsection: SMD framework}
In this section, we state the SMD framework and then briefly describe how we use it to compute an (expected) $\varepsilon$-approximate solution to the optimization problem in Equation \ref{equation: the final min maz optimization}. This is done using the ghost iterate technique \citep{jin2020efficiently} to compute the approximate solution using only stochastic query access to $\mathbf{\Gamma}$ and $\rvec$ (i.e. the generative model). We begin with two useful definitions.

\begin{definition}{(Strong Convexity)}
Let $\mathcal{X} \subset \mathbb{R}^{n}$ be a convex set. A differentiable function $\R: \mathcal{X} \rightarrow \mathbb{R}$ is said to be $\alpha$-strongly convex with respect to norm $||.||$ if $\R(\yvec) \geq \R(\xvec) + \langle \nabla \R(\xvec), \yvec-\xvec \rangle + \frac{\alpha}{2} ||\yvec - \xvec ||^2 $ for all $ \xvec, \yvec \in \mathcal{X}$.
\label{def:strongConvexity}
\end{definition}

\begin{definition}{(Distance Generating Function)}
Let $\R: \mathcal{X} \rightarrow \mathbb{R}$ be a continuously differentiable, strongly convex, real-valued function on a convex set $\mathcal{X} \subset \mathbb{R}^{n}$. For any $\xvec,\yvec \in \mathcal{X}$, the distance from point $\xvec$ to $\yvec$ is given by $ V_{\xvec}(\yvec) = \R(\yvec) - \R(\xvec) - \langle \nabla \R(\xvec), \yvec -\xvec \rangle.$\footnote{This function is also called as the Bregman divergence or the prox-function. }
\label{def:distGenFunction}
\end{definition}
The SMD framework is a stochastic approximation approach of finding a solution to a stochastic convex program \cite{nemirovski2009robust}. In particular, the SMD algorithm is a special type of stochastic gradient descent (SGD) algorithm where the updates are computed in the \emph{mirrored} space. It provides an iterative procedure to select points from a convex space $\mathcal{X}$ with stochastic query access to underlying parameters. First, a suitable strongly convex regularizer $R$ catering to the geometry of $\mathcal{X}$ is designed. Then, at each iterate $t$, an unbiased estimator of gradient $\gvec$ given by $\tilde{\gvec}_{t}$ and a step size $\eta_{t}$ is computed. This estimator, along with the previous iterate value $\xvec_{t}$ and step size $\eta_{t}$, is used to compute the next iterate as given below: 
\begin{equation}
 \xvec_{t+1} = \arg \min_{x \in \mathcal{X}} ~~\langle \eta_t \tilde{\gvec}_{t} , \xvec \rangle + V_{\xvec_{t}}(\xvec)~.
 \label{eqn:SMD update}
\end{equation}
Since the regularizer is strictly convex and differentiable, the SMD update can be equivalently written in the (mirrored) gradient space as follows 
\begin{equation*}
 \nabla R(\xvec_{t+1}) = \nabla R(\xvec_t) - \eta_t \tilde{\gvec}_{t}~. 
\end{equation*}
The strong convexity of the regularizer $R$ implies the uniqueness of the mapping of $\nabla R(\cdot)$. For more details, the reader is referred to the works by \citet{nemirovski2009robust} and \citet{carmon2019variance}. 

In this paper, we use the SMD framework to obtain an approximate solution to the optimization problem defined in Equation \ref{equation: the final min maz optimization}. Note that, in our problem, there are two convex spaces: $\Delta_{\rhovec}^{\ell}$ corresponding to $\xvec$, and $\mathbb{B}_{2M}^{\ell}$ corresponding to $\lambdavec$. First, we choose strongly convex regularizers for the respective spaces and initialize the \emph{constant} step sizes $\eta^{\xvec}$ and $\eta^{\lambdavec}$, respectively. Then, at each iterate, we compute bounded estimators of the gradients $\gvec^{\xvec}$ and $\gvec^{\lambdavec}$ given by $\tilde{\gvec}_{t}^{\xvec}$ and $\tilde{\gvec}_{t}^{\lambdavec}$, respectively. Here, $\gvec^{\xvec}$ and $\gvec^{\lambdavec}$ are the gradients of $f$ with respect to $\xvec$ and $\lambdavec$, respectively. Finally, we use the SMD update rule given in Equation \ref{eqn:SMD update} to compute the iterate value for the respective spaces. In this paper, we consider bounded gradient estimates of the gradients as defined below.

\begin{definition}{(\citet{jin2020efficiently})} \label{definition: bounded estimators}
Given following properties on mean, scale, and variance of an estimator $\tilde{\gvec}$ of the gradient $\gvec$:
\begin{enumerate}[leftmargin=*]
\item unbiasedness: $\mathbb{E}[\tilde{\gvec}] = \gvec$,
\item bounded maximum entry: $||\tilde{\gvec}||_{\infty} \leq c$ with probability $1$, and
\item bounded second moment $\mathbb{E}[||\tilde{\gvec}||^2] \leq v$,
\end{enumerate}
we say that $\tilde{\gvec}$ is a $(v, ||.||)$-bounded estimator of $\gvec$ if it satisfies $(1)$ and $(2)$, and that it is a $(c, v, ||.||_{\Delta_{\rhovec}^{\ell}})$-bounded estimator if it satisfies all three with local norm $||.||_{\yvec}$ for all $\yvec \in \Delta_{\rhovec}^{\ell}$. 
\end{definition}
Next, we present our algorithm and the bounded gradient estimators and regularizers corresponding to the two convex spaces (mentioned above) used in the algorithm.

\section{Algorithm and Estimators}
\label{sec:algorithm}
\begin{algorithm}[h]
\SetAlgoLined
\DontPrintSemicolon
\kwInput{Desired accuracy $\varepsilon$}
\kwOutput{An $\varepsilon$-approx policy $\pi^{\varepsilon}$} 
\kwParam{
$\eta^{\b{\lambda}} \leq \varepsilon / 16$, 
$\eta^{\b{x}} \leq \varepsilon / (8\ell(24M^2+1))$, 
$T \geq \max\left(8\log \ell / (\eta^{\xvec}\varepsilon), 32M^2n / (\eta^{\lambdavec} \varepsilon)\right)$
}
\For{$t = 1, \ldots , T$} {
Get $\tilde{\gvec}^{\b{\lambda}}_t$ as a $((4M+1)\ell, (24M^2+1) \ell,||.||_{\Delta^{\ell}})$-bounded estimator of $\gvec^{\b{\lambda}}(\b{x}_t, \b{\lambda}_t)$\;
Get $\tilde{\gvec}^{\b{x}}_t$ as a $(2,||.||_2)$-bounded estimator of $\gvec^{\b{x}}(\b{x}_t, \b{\lambda}_t)$\;
Let $\b{x}_{t+1} = \arg\min_{\b{x} \in \Delta_{\b{\rho}}^{\ell}} \langle \eta^{\b{x}}\tilde{\gvec}^{\b{x}}_t, \b{x}\rangle + V_{\b{x}_t}(\b{x})$\;
Let $\b{\lambda}_{t+1} = \arg\min_{\b{\lambda} \in \mathbb{B}_{2M}^n} \langle \eta^{\b{\lambda}}\tilde{\gvec}^{\b{\lambda}}_t, \b{\lambda}\rangle + V_{\b{\lambda}_t}(\b{\lambda})$\;
}
Let $( \b{x}^{\varepsilon},\b{\lambda}^{\varepsilon}) = \frac{1}{T}\sum_{t=1}^T (\b{x}_t, \b{\lambda}_t )$\;
\Return{$\pi^{\varepsilon}$ with $\pi^{\varepsilon}(a|s) = x_{s,a}/\sum_{a'=1}^m x_{s,a'}$}
\caption{Fair State-Visitation} \label{algorithm: fair state visitation}
\end{algorithm}

In Section \ref{sec: constrained domain and SMD}, we formulated Fair-AMDP as a bilinear saddle point problem using the Lagrangian function of \textsc{Fair-LP}. We also showed how the SMD framework can be used to compute an expected approximate solution to this bilinear problem. Given input $\varepsilon$, Algorithm \ref{algorithm: fair state visitation} computes $(\xvec^{\varepsilon}, \lambdavec^{\varepsilon})$, which is an expected $\varepsilon$-approximate solution to the optimization problem in Equation \ref{equation: the final min maz optimization}, i.e. $\mathbb{E}[\Gap(\xvec^{\varepsilon}, \allowbreak \lambdavec^{\varepsilon})] \allowbreak \leq \varepsilon$. The details of the regularizers and estimators are provided in the next paragraph.
Then, in Line 8, the algorithm uses $\xvec^{\varepsilon}$ to compute $\pi^{\varepsilon}$. Note that it is not immediately clear that $\pi^{\varepsilon}$ satisfies the approximation guarantees (in expectation) as mentioned in Definition \ref{definiton: approximate policy} even though $\xvec^{\varepsilon}$ is $\varepsilon$-approximate solution for the corresponding bilinear saddle point problem. In addition, note that Algorithm \ref{algorithm: fair state visitation} is oblivious to the presence of the fair-action.

\begin{table}[!b]
\centering
 \begin{tabular}{l l l}
 \toprule
   & $R$ & $V_{\xvec'}(\xvec) $ \\ \midrule 
 $\xvec$ - space & $\sum_{s,a} x_{s,a} \log(x_{s,a})$ & $\sum_{s,a} x'_{s,a} \log(x'_{s,a}/x_{s,a}) $ \\
 $\lambdavec$ - space & $\frac{1}{2}||\lambdavec||_2^{2}$ & $\frac{1}{2}||\lambdavec - \lambdavec'||_{2}^{2}$ \\ \bottomrule
 \end{tabular}
 \vspace{0.05in}
 \caption{SMD parameter setting}
 \vspace{-0.15in}
 \label{table: regularizer and distance}
\end{table}

We use normalized entropic regularizer for $\Delta_{\rhovec}^{\ell}$, and $||.||_{2}^2$ regularizer for $\mathbb{B}_{2M}^{\ell}$, as shown in Table \ref{table: regularizer and distance}. The gradients of $f$ with respect $\xvec$, and $\lambdavec$ are given as 
\begin{align*}
    \gvec^{\xvec} &= (\mathbf{\Gamma} - \hat{\mathbf{I}})\lambdavec - \rvec, \\
    \gvec^{\lambdavec} &= (\hat{\mathbf{I}} - \mathbf{\Gamma})^T\xvec. 
\end{align*}
Note that since we maximize with respect to $\xvec$, we take the negative of the gradient with respect to $\xvec$. To estimate $\tilde{\gvec}^{\xvec}_{t}$ and $\tilde{\gvec}^{\lambdavec}_{t}$ of $\gvec^{\xvec}$ and $\gvec^{\lambdavec}$, we use the generative model as follows, where $(s,a) \sim \xvec_t$ denotes the state-action pair sampled from the distribution $\xvec_t \in \Delta_{\rhovec}^{\ell}$.

\begin{enumerate}[leftmargin=*]
    \item 
    For an $\xvec \in \Delta^{\ell}_{\rhovec}$, the estimator $\tilde{\gvec}_{t}^{\lambdavec}$ for $\gvec^{\b{\lambdavec}}$ in Line 2 is computed: 
    \begin{align*}
    (s,a) &\sim \b{x}, \,\,\,\, s' \sim \mathbf{\Gamma}((s,a),s')\\ \tilde{\gvec}^{\b{\lambda}}(\xvec,\lambdavec) &= \b{e}_{s} - \b{e}_{s'} ~,
    \end{align*}
    where $\b{e}_s$ is the unit vector in $\mathbb{R}^n$. Finally, $\tilde{\gvec}^{\b{\lambda}}_t = \tilde{\gvec}^{\b{\lambda}}(\xvec_{t-1},\lambdavec_{t-1})$.
    \item 
    The estimator $\tilde{\gvec}_{t}^{\xvec}$ of $\gvec^{\b{x}}$ in Line 3 is computed as: 
    \begin{align*}
    (s,a) &\sim \left[1/\ell\right], \,\,\,\, s' \sim \mathbf{\Gamma}((s,a),s')\\
    \tilde{\gvec}^{\b{x}}(\xvec,\lambdavec) &= \ell(\lambda_{s'} - \lambda_{s} - r_{s,a})\b{e}_{s,a}, 
    \end{align*}
    where $\lambda_{s}$ denotes the $s$-th entry of $\lambdavec$, $\b{e}_{s,a}$ is a unit vector in $\mathbb{R}^{\ell}$, and $\tilde{\gvec}^{\b{x}}_t = \tilde{\gvec}^{\b{x}}(\xvec_{t-1},\lambdavec_{t-1})$.
\end{enumerate}

\noindent Next, we show that $\tilde{\gvec}^{\b{x}}_t$ and $\tilde{\gvec}^{\b{\lambda}}_t$ are bounded (Definition \ref{definition: bounded estimators}).

\begin{restatable}{lemma}{XBoundedEstimator}
\label{lemma: bounded estimators for x}
The estimator $\tilde{\gvec}^{\b{x}}(\xvec,\lambdavec)$ as constructed above is a $((4M+1)\ell, (24M^2+1) \ell,||.||_{\Delta^{\ell}})$-bounded estimator.
\end{restatable}

\begin{proof}
From the definition of $\tilde{\gvec}^{\b{x}}(\xvec,\lambdavec)$ it follows that 
\begin{align*}
 \mathbb{E}[\tilde{\gvec}^{\b{x}}] &= \sum_{s, s' \in [n], a\in [m]} \mathbf{\Gamma}((s,a),s)(\lambda_{s'} - \lambda_{s} - r_{s,a})e_{s,a} \\
 &= (\mathbf{\Gamma} - \Ihat)\b{\lambda} - \b{r}.
\end{align*}
Since $||\lambdavec||_{\infty} \leq 2M$ and $\rvec \in [0,1]^{\ell}$, $|\lambda_{s'} - \lambda_s - r_{s,a}| \leq 4M+1$, we have $||\tilde{\gvec}^{\b{x}}_t||_{{\infty}} \leq (4M+1)\ell$. For the second moment, we have that for any $\xvec' \in \Delta^\ell$, 
\begin{align*}
    \mathbb{E}[||\tilde{\gvec}^{\b{x}} ||_{\b{x}'}^2] \leq \sum_{s,a} \frac{1}{\ell} x'_{s,a} (24M^2+1) \ell^2=(24M^2+1)\ell. 
\end{align*}
\end{proof}

\begin{restatable}{lemma}{LambdaBoundedEstimator}\label{lemma: bounded estimators for lambda and q}
The estimator $\tilde{\gvec}^{\b{\lambda}}(\xvec,\lambdavec)$ as constructed above is a $(2, \allowbreak ||.||_2)$ bounded estimator.
\end{restatable}
The proof of Lemma \ref{lemma: bounded estimators for lambda and q} is in Appendix 
\ifdefined\aamas
A.
\else
\ref{secappendix: proof of lemma 4}.
\fi
We note that if $\b{\rho} = \b{0}$, then our algorithm is the same as that by \citet{jin2020efficiently} for the unconstrained AMDP. In Section \ref{sec: theoretical results}, Theorem \ref{thm:SMDGapTheorem} shows that $(\xvec^{\varepsilon}, \lambdavec^{\epsilon})$ is such that $\mathbb{E}\Gap(\xvec^{\varepsilon}, \lambdavec^{\varepsilon}) \leq \varepsilon$. Then, Theorem \ref{thm: optimality and fairness} shows that the policy $\pi^{\varepsilon}$ constructed at Line 8 is indeed $(3\varepsilon, \varepsilon)$-approximate in expectation, and thus is the main contribution of our work. 

\section{Theoretical Results}\label{sec: theoretical results}
Theorem \ref{thm:SMDGapTheorem} shows that $(\xvec^{\varepsilon}, \boldsymbol{\lambda}^{\varepsilon})$ computed by Algorithm \ref{algorithm: fair state visitation} at Line 5 satisfies $\mathbb{E}[\Gap(\xvec^{\varepsilon}, \lambdavec^{\varepsilon})] \leq \varepsilon$. 
\begin{restatable}{theorem}{FirstThm}
Given a state-visitation fairness vector $\rhovec \in [0,1/n)^{n} $, desired accuracy $\varepsilon >0$, and bounded estimators $\tilde{\gvec}_t^{\xvec}$ and $\tilde{\gvec}_t^{\lambdavec}$ as given in Lemmas \ref{lemma: bounded estimators for x} and \ref{lemma: bounded estimators for lambda and q}, Algorithm \ref{algorithm: fair state visitation} with step sizes $\eta^{\xvec} \leq \frac{\varepsilon}{8\ell(24M^2+1)}$ and $\eta^{\lambdavec} \leq \varepsilon/16$, at the end of $T \geq \max\left(\frac{8\log \ell}{\eta^{\xvec}\varepsilon}, \frac{32M^2n}{\eta^{\lambdavec} \varepsilon}\right)$ rounds, at Line 5 computes $(\boldsymbol{x}^{\varepsilon}, \boldsymbol{\lambda}^{\varepsilon})$ such that $\mathbb{E}[\Gap(\xvec^{\varepsilon},\lambdavec^{\varepsilon})] \leq \varepsilon$ . 
\label{thm:SMDGapTheorem}
\end{restatable}
The expectation is over the stochasticity in both the state-action selection and the MDP transitions. The proof of Theorem \ref{thm:SMDGapTheorem} is in Appendix 
\ifdefined\aamas
B.
\else
\ref{sec: proof of theorem 2}.
\fi
We make two remarks. First, the step size $\eta^{\xvec}$ and time horizon $T$ depend on the mixing time of an MDP. The problem of estimating (or providing an upper bound) on the mixing time of an MDP is an active research topic \cite{paulin2015,zahavy2019average}. We rely on these techniques to estimate the mixing time used to compute step size and time horizon.

Second, the lower bound on $T$ is proportional to $M$, which becomes very large as $\rho_{s}$ approaches $1/n$. Hence, if $\rho_s$'s are close to $1/n$ then the search space of $\boldsymbol{\lambda}$ in Algorithm \ref{algorithm: fair state visitation} increases and it converges slowly. Note that, if $\rho_s \leq 1/2n$ then $d_{\boldsymbol{\rho}} \leq \max_s \frac{1}{\rho_s}$ then our algorithm yields similar convergence time (in order) as that of the constrained AMDP algorithm by \cite{jin2020efficiently} that only computes a feasible solution without reward maximization. This bound, without Assumption \ref{asm:fairAction}, can be obtained from a stationary distribution of (any) strictly feasible policy. 

Next, we define the notion of simultaneous approximation guarantee of an solution to the Fair-AMDP problem and show that the the policy $\pi^{\varepsilon}$ computed by Algorithm \ref{algorithm: fair state visitation} satisfies this notion in Theorem \ref{thm: optimality and fairness}.

\begin{definition}
\label{definiton: approximate policy}
Let $\xvec^*$ be the optimal solution to \textsc{Fair-LP (D)}. A policy $\pi$ is called $(\varepsilon_1,\varepsilon_2)$-approximate for the Fair-AMDP problem if $\rvec^T (\Pi\nuvecpi) \geq \rvec^T\xvec^* - \varepsilon_1$ and $\mathbf{C}\xvec \geq \b{1} -\varepsilon_2$.
\end{definition}

\begin{restatable}{theorem}{Optimalityfairnessthm}
Given a state-visitation fairness vector $\rhovec \in \left[0, 1/n\right)^n$, desired accuracy $\varepsilon >0$ and bounded estimators $\tilde{g}^{\xvec}_t$ and $\tilde{g}^{\lambdavec}_t$ as given in Lemmas \ref{lemma: bounded estimators for x} and \ref{lemma: bounded estimators for lambda and q}, Algorithm \ref{algorithm: fair state visitation} with step sizes and $T$ as in Theorem \ref{thm:SMDGapTheorem} computes a policy $\pi^{\varepsilon}$ which is $(3\varepsilon,\varepsilon)$-approximate in expectation with sample complexity $O\left(nm \varepsilon^{-2} t_{mix}^{2}(1 + d_{\rhovec})^2 \log(nm)\right)$. 
\label{thm: optimality and fairness}
\end{restatable}
The proof of Theorem \ref{thm: optimality and fairness} is given in Appendix 
\ifdefined\aamas
C.
\else
\ref{secappendix: proof of fairness and optimality}.
\fi
As stated earlier, Algorithm \ref{algorithm: fair state visitation} is oblivious to the presence of fair-action. Assumption \ref{asm:fairAction} is not necessary if strict feasibility (even for $\boldsymbol{\rho} \in [0,1]^{n}$ such that $\sum_s \rho_s <1$) is known \textit{a priori}. The convergence time $T$ in this case would depend on the reward and the stationary distribution induced by the strictly feasible policy. We note an additional multiplicative factor of $(1 + d_{\rhovec})^2$ in the sample complexity of our proposed algorithm when compared with the best known sample complexity result for the unconstrained problem (see Theorem 1 by \citet{jin2020efficiently}). In the proposed framework, this factor is required to ensure that the bounding box contains the optimal value of the primal variable $\lambdavec$ and is the price we pay to compute a policy with simultaneous approximation guarantee on both fairness and reward. However, the sample complexity is not tight when $\rhovec = \boldsymbol{0}$. In this case $(1 + d_{\rhovec})^2 = n^2$ and the sample complexity of the proposed algorithm becomes $O(n^3 m \varepsilon^{-2} t_{mix}^2 \log(nm))$ which is suboptimal. Finding a better sample complexity for fair-AMDP problem is an interesting future work.

\section{Simulations}\label{sec:experiments}
\begin{figure*}[htb]
    \tikzset{every picture/.style={line width=0.75pt}} 

\begin{tikzpicture}[scale = 0.65, x=0.75pt,y=0.75pt,yscale=-1,xscale=1]

\draw   (161,98) .. controls (161,88.06) and (169.06,80) .. (179,80) .. controls (188.94,80) and (197,88.06) .. (197,98) .. controls (197,107.94) and (188.94,116) .. (179,116) .. controls (169.06,116) and (161,107.94) .. (161,98) -- cycle ;
\draw   (241,98) .. controls (241,88.06) and (249.06,80) .. (259,80) .. controls (268.94,80) and (277,88.06) .. (277,98) .. controls (277,107.94) and (268.94,116) .. (259,116) .. controls (249.06,116) and (241,107.94) .. (241,98) -- cycle ;
\draw   (320,98) .. controls (320,88.06) and (328.06,80) .. (338,80) .. controls (347.94,80) and (356,88.06) .. (356,98) .. controls (356,107.94) and (347.94,116) .. (338,116) .. controls (328.06,116) and (320,107.94) .. (320,98) -- cycle ;
\draw [color={rgb, 255:red, 74; green, 144; blue, 226 }  ,draw opacity=1 ][line width=1.5]    (179,80) .. controls (200.12,61.76) and (230.46,60.11) .. (255.85,77.7) ;
\draw [shift={(259,80)}, rotate = 217.57] [fill={rgb, 255:red, 74; green, 144; blue, 226 }  ,fill opacity=1 ][line width=0.08]  [draw opacity=0] (13.4,-6.43) -- (0,0) -- (13.4,6.44) -- (8.9,0) -- cycle    ;
\draw [color={rgb, 255:red, 74; green, 144; blue, 226 }  ,draw opacity=1 ][line width=1.5]    (259,80) .. controls (280.12,61.76) and (309.53,60.11) .. (334.86,77.7) ;
\draw [shift={(338,80)}, rotate = 217.57] [fill={rgb, 255:red, 74; green, 144; blue, 226 }  ,fill opacity=1 ][line width=0.08]  [draw opacity=0] (13.4,-6.43) -- (0,0) -- (13.4,6.44) -- (8.9,0) -- cycle    ;
\draw [color={rgb, 255:red, 245; green, 166; blue, 35 }  ,draw opacity=1 ][line width=1.5]    (262.53,118.19) .. controls (287.69,132.88) and (316.04,131.28) .. (338,116) ;
\draw [shift={(259,116)}, rotate = 33.18] [fill={rgb, 255:red, 245; green, 166; blue, 35 }  ,fill opacity=1 ][line width=0.08]  [draw opacity=0] (13.4,-6.43) -- (0,0) -- (13.4,6.44) -- (8.9,0) -- cycle    ;
\draw [color={rgb, 255:red, 245; green, 166; blue, 35 }  ,draw opacity=1 ][line width=1.5]    (182.54,118.19) .. controls (207.77,132.88) and (237.04,131.28) .. (259,116) ;
\draw [shift={(179,116)}, rotate = 33.18] [fill={rgb, 255:red, 245; green, 166; blue, 35 }  ,fill opacity=1 ][line width=0.08]  [draw opacity=0] (13.4,-6.43) -- (0,0) -- (13.4,6.44) -- (8.9,0) -- cycle    ;
\draw [color={rgb, 255:red, 74; green, 144; blue, 226 }  ,draw opacity=1 ][line width=1.5]    (356,98) .. controls (429.26,177.2) and (88.92,179.95) .. (158.76,100.43) ;
\draw [shift={(161,98)}, rotate = 493.93] [fill={rgb, 255:red, 74; green, 144; blue, 226 }  ,fill opacity=1 ][line width=0.08]  [draw opacity=0] (13.4,-6.43) -- (0,0) -- (13.4,6.44) -- (8.9,0) -- cycle    ;
\draw [color={rgb, 255:red, 245; green, 166; blue, 35 }  ,draw opacity=1 ][line width=1.5]    (161,98) .. controls (87.91,18.65) and (428.25,16.21) .. (358.25,95.57) ;
\draw [shift={(356,98)}, rotate = 314.05] [fill={rgb, 255:red, 245; green, 166; blue, 35 }  ,fill opacity=1 ][line width=0.08]  [draw opacity=0] (13.4,-6.43) -- (0,0) -- (13.4,6.44) -- (8.9,0) -- cycle    ;

\draw (170,91) node [scale=1][anchor=north west][inner sep=0.75pt]    {$s_{0}$};
\draw (250,91) node [scale=1][anchor=north west][inner sep=0.75pt]    {$s_{1}$};
\draw (329,91) node [scale=1][anchor=north west][inner sep=0.75pt]    {$s_{2}$};
\end{tikzpicture} 
    \includegraphics[width=0.32\textwidth]{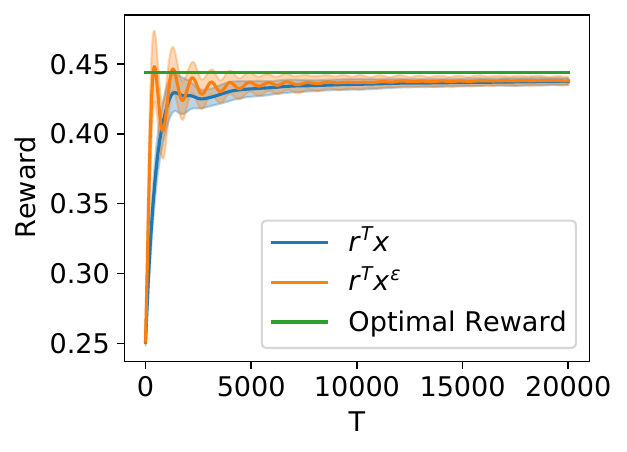}
    \includegraphics[width=0.32\textwidth]{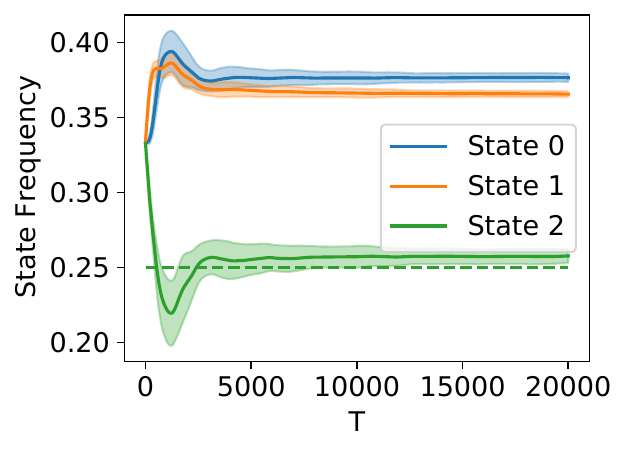}
    \caption{Left: the MDP specification, with three states and two actions. In each state, action $a_0$ takes the blue transition with probability  0.9, and takes the yellow transition with probability  0.1; action $a_1$ has the converse effect. Middle and right: reward and state-visitation frequency of a policy learned after $T$ steps, averaged over 100 runs. }
    \label{fig:1-3}
\end{figure*}
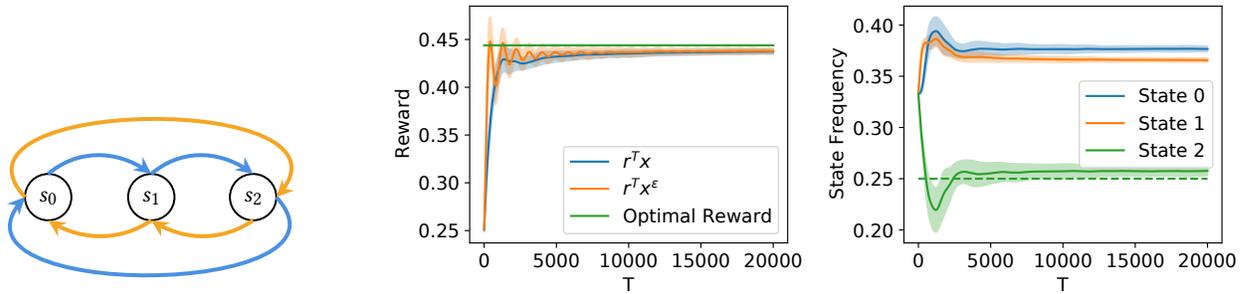
\begin{figure*}[htb]
    \centering
    \includegraphics[scale=0.5]{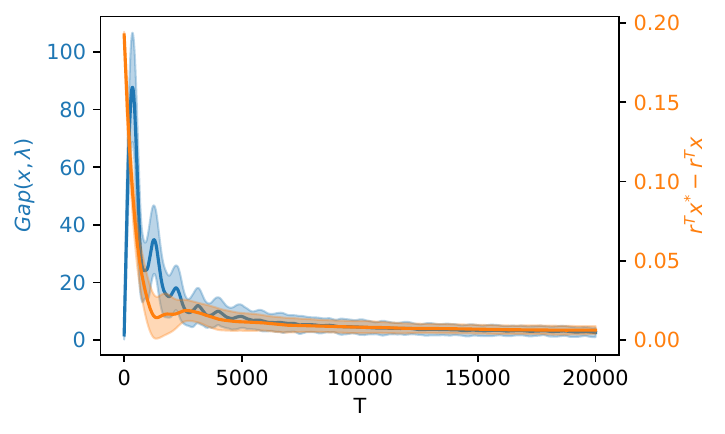}
    \includegraphics[width=0.32\textwidth]{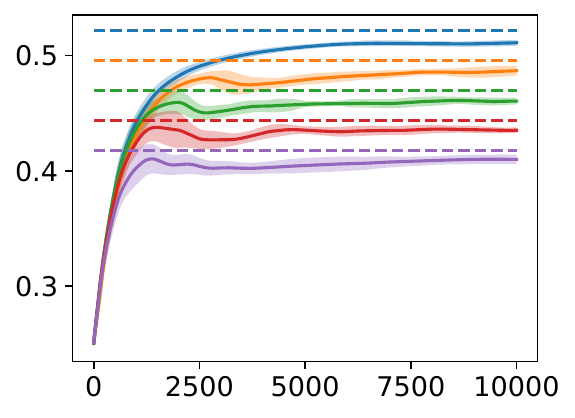}
    \includegraphics[width=0.32\textwidth]{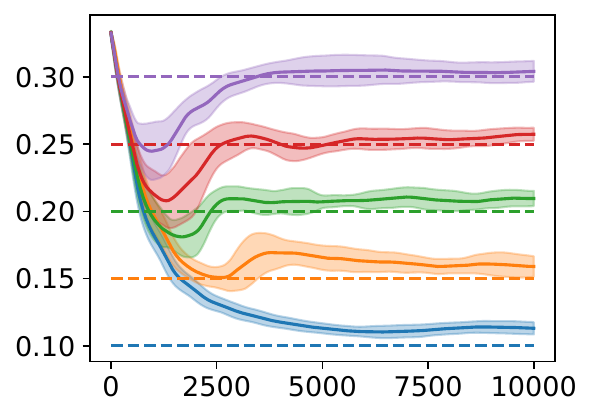}
    \caption{Left: gap function value (blue) and reward difference (orange) learned after $T$ steps, averaged over 100 runs. Middle and right: reward and visitation frequency of $s_2$ of policies learned after $T$ steps, averaged over 10 runs, for each $\rho_2\in \{0.1, 0.15, 0.2, 0.25, 0.3\}$. The colors of blue, orange, green, red, and purple indicate increasing values of $\rho_2$. Dashed lines indicate optimal rewards and visitation frequency constraints. }
    \label{fig:rho}
\end{figure*}

In this section, we evaluate our proposed algorithm and validate its  properties with experiments on simulated  data. Figure \ref{fig:1-3} presents the MDP transition dynamics specification with three states, $s_0, s_1, s_2$, and two actions, $a_0, a_1$. The reward of taking $a_0$ on $s_0$ is 1, and 0.1 otherwise. It is easy to see that the optimal unconstrained policy is a deterministic policy with $\pi(s_0) = a_0$, $\pi(s_1) = a_1$, and $\pi(s_2) = a_0$. This policy yields an average-reward of 0.526, with the three states $s_0, s_1$, and $s_2$ being visited 47.4\%, 43.5\%, and 9.1\% of times respectively. Also, since the MDP is feasible for $\rhovec \leq 1 / 3$, we do not use the fair-action. Recall, $\xvec^{\varepsilon}$ is the average of $\xvec_t$, $t\in [T]$. Further, we use $\xvec$ to denote the state-action frequency vector of policy $\pi^{\varepsilon}$ derived from $\xvec^{\varepsilon}$ (Step 8, Algorithm \ref{algorithm: fair state visitation}). Note that the reward of $\pi^{\varepsilon}$ is $\rvec^T\xvec$.

\textbf{Convergence of SMD: }We first empirically demonstrate the convergence of SMD. We choose a fairness constraint $\rho=[0.1, \allowbreak 0.1, \allowbreak 0.25]$, which is violated by the unconstrained optimal policy. 
We implemented Algorithm \ref{algorithm: fair state visitation} to solve this problem with $M=100$ and  $\eta^{\xvec}=\eta^{\lambdavec}=0.01$. Figure \ref{fig:1-3} shows the average-reward and state-visitation frequencies for policies obtained after different number of gradient descent steps, up to a maximum of 20,000. The results are averaged over 100 runs, with the shaded regions representing the standard deviations. We also see that the reward $\rvec^T \xvec$, approaches the fair optimal value and the visitation frequency of $s_2$  approaches the required value (25\%). We note that $\rvec^T\xvec$ can cross the optimal value as $\xvec$ is only approximately fair. In Figure \ref{fig:1-3}, we also plot $\rvec^T \xvec^{\varepsilon}$ along with $\rvec^T \xvec$ for different values of $T$.

\textbf{Gap Function: }The theorems prove that the difference between $\rvec^T \xvec$ and the optimal reward decreases with the value of the the gap function \texttt{Gap}$(\xvec^\varepsilon, \lambdavec^\varepsilon)$. Figure \ref{fig:rho} (left) plots the reward difference as the gap decreases.

\textbf{Varying the Constraint $\rhovec$: }In this section, we vary the value of $\rho_2\in\{0.1, 0.15, 0.2, 0.25, 0.3\}$. For each $\rhovec$ parameter, we run SMD 10 times, with 10,000 steps each. Figure \ref{fig:rho} plots the reward and state-visitation frequency on $s_2$ for different $\rhovec$. The colors blue, orange, green, red, and purple indicate increasing values of $\rho_2$. 
As expected, all fairness constraints are satisfied, but as $\rho_2$ increases, the optimal rewards decreases. This decrease captures the cost of fairness as the fairness constraints become more stringent.

\section{Conclusion and Future Work}
\label{sec:discussion}
In this paper, we studied the Fair-AMDP problem and proposed an SMD-based algorithm to compute a policy with simultaneous approximation guarantee on average-reward and state-visitation fairness. We note that our results in this work can be extended to ensure state-action visitation fairness: for $\rhovec \in [0,1]^{\ell}$ a fair policy $\pi$ with stationary distribution $\nu^{\pi}$ satisfies $\Pi\cdot \nu^{\pi} \geq \rhovec$, provided there is a strictly feasible solution.
 
In Section \ref{subsection: Min-Max Problem}, we determined the domain for the primal and dual variables for approximately solving the bilinear saddle point problem using the SMD framework. In particular, we restricted our domain of $\xvec$ to $\Delta_{\rhovec}^{\ell}$ and we removed $\muvec$ variables from the optimization problem. This in turn demands that at step 4 in Algorithm \ref{algorithm: fair state visitation}, $\xvec_t$ is computed by projecting to $\Delta_{\rhovec}^{\ell}$. This approach can also be adapted for computing a feasible solution for an AMDP with arbitrary linear constraints and no objective function \citep{jin2020efficiently}. 

For future work, the time complexity bound of our algorithm becomes large as $\rho_s \rightarrow 1/n$. Thus, an immediate direction is to improve the dependence of time complexity on $\rhovec$, either via a new algorithm or better analysis.
Also observe that the introduction of fair-action ensures strict feasibility when $\rho_s < 1/n$ for all $s$. It is interesting to see if there is an approach that ensures feasibility for a broader class of $\rhovec$.

We note that a convex program formulation for discounted-reward MDPs with state-visitation fairness constraints is not immediate, and we need novel techniques to solve the state-visitation fairness problem for discounted-reward MDPs. The introduction of fairness constraints may lead to a loss in the maximum average-reward that can be achieved. The difference in the average-reward of an unconstrained optimal policy and that of an optimal policy that satisfies fairness captures the \emph{price of fairness} in the system. Further study of price of fairness in the Fair-AMDP setting is an interesting future direction.

\begin{acks}
We thank the anonymous reviewers for their reviews. Vineet Nair is thankful to be supported by the European Union's Horizon 2020 research and innovation program under grant agreement No 682203 -ERC-[Inf-Speed-Tradeoff]. Vishakha Patil is grateful for the support of a Google PhD Fellowship.
\end{acks}

\balance
\bibliographystyle{ACM-Reference-Format} 
\bibliography{references}

\ifdefined\aamas
\else
\onecolumn 
\appendix

\section{Proof of Lemma \ref{lemma: bounded estimators for lambda and q}}
\label{secappendix: proof of lemma 4}
\LambdaBoundedEstimator*
\begin{proof}
First we show that $\tilde{\gvec}^{\boldsymbol{\lambda}}$ is a $(2,||.||_2)$ bounded estimator. From the definition of $\tilde{\gvec}^{\boldsymbol{\lambda}}(\xvec,
\lambdavec)$ it follows that
\begin{align*}
\mathbb{E}[\tilde{\gvec}^{\boldsymbol{\lambda}}] &= \sum_{s,s' \in [n], a\in [m]} x_{s,a} \mathbf{\Gamma}((s,a), s')(e_{s} - e_{s'}) \\  &= (\Ihat-\mathbf{\Gamma})^T\boldsymbol{x}~.
\end{align*}
For the bound on the second-moment, observe that $||\tilde{\gvec}^{\boldsymbol{\lambda}}(\xvec,\lambdavec)||_2^2 \leq 2$ for any $\xvec, \lambdavec$. 
\end{proof}

\section{ Proof of Theorem \ref{thm:SMDGapTheorem}}
\
\label{sec: proof of theorem 2}
\FirstThm*
 
First, in Lemma \ref{lem:gapLemma},  we show that the $\Gap$ function can be upper bounded in terms of the sequences of gradients  ($\{\gvec_{t}^{\xvec}\}_{t=1}^{T}$ and $\{\gvec_{t}^{\lambdavec}\}_{t=1}^{T}$) and the sequence of  choices of algorithm ($\{\xvec_{t}\}_{t=1}^{T}$ and $\{\lambdavec_{t}\}_{t=1}^{T}$) as follows. 
\begin{restatable}{lemma}{GapLemma}
Let $\{\xvec_{t}, \lambdavec_t \}_{t=1}^{T}$ be a sequence of the iterate values of Algorithm \ref{algorithm: fair state visitation} and $(\xvec^{\varepsilon}, \lambdavec^{\varepsilon} ) := \frac{1}{T} \sum_{t=1}^{T} (\xvec_{t}, \lambdavec_{t} )$. Then we have
\begin{equation*}
\small
    \Gap(\xvec^{\varepsilon}, \lambdavec^{\varepsilon} ) \leq  \frac{1}{T} \sup_{\xvec, \lambdavec }  \sum_{t=1}^{T} [ \langle \gvec_t^{\xvec}, \xvec_t - \xvec_{T+1} \rangle  + \langle \gvec_t^{\lambdavec}, \lambdavec_t - \lambdavec_{T+1} \rangle  ] 
\end{equation*}
\label{lem:gapLemma}
\end{restatable}
Next, we bound the individual terms in the statement of  Lemma \ref{lem:gapLemma} using following supporting Lemmas.  
\begin{restatable}{lemma}{Name}[Lemma 12 \cite{carmon2019variance}]
Let $\mathcal{X}$ be a non-empty compact convex  and  $\R$ be a   $1$-convex regularizer   with respect to   $||.||$.  Further, let  $\mathcal{X}^{\star}$ be the dual space of  $\mathcal{X}$ and the dual norm of a  vector $\boldsymbol{\gamma} \in \mathcal{X}^{\star}$ be defined as $||\boldsymbol{\gamma}||_{\star} := \sup_{ ||\lambdavec||\leq 1} \langle \boldsymbol{\gamma} , \lambdavec \rangle$. Given a sequence  $\{\boldsymbol{\gamma}_{t}\}_{t = 1}^{T} \in \mathcal{X}^{\star}$, 
let the sequence $\{\lambdavec_{t}\}_{t=1}^{T} \in \mathcal{X}$ be defined as  $$\lambdavec_{t} = \arg\min \limits_{\lambdavec \in \mathcal{X}} \langle \gammavec_{t-1}, \lambdavec \rangle + V_{\lambdavec_{t-1}}(\lambdavec). $$
Then, for any time instant $T \geq 2$, for every choice of   $\lambdavec_0\in \mathcal{X}$  we have, 
\begin{equation*}
    \sum_{t =1 }^{T} \langle \boldsymbol{\gamma}_t, \lambdavec_t - \lambdavec_{T+1} \rangle \leq V_{\lambdavec_0}(\lambdavec_{T+1}) + 1/2 \sum_{t =1 }^{T} ||\gammavec_t||_{\star}^2 
\end{equation*}
\label{lem:Lem12Carmon}

\end{restatable}

\begin{restatable}{lemma}{Two}[Lemma 13 \cite{carmon2019variance}]
Let $R$ be an entropic regularizer defined over $\Delta_{\rhovec}^{\ell}$. Also, let a  sequence  $\{\boldsymbol{\delta}_{t}\}_{t=1}^{T} \in \mathbb{R}_{+}^{\ell} $ be such that $\delta_{ti} \leq 1.79$ for all  $ t \in [T]$ and $i \in [\ell]$. Let the sequence $\{\xvec_{t}\}_{t=1}^{T} \in \mathcal{X}$ is generated as 

$ \xvec_{t} = \arg \min_{\xvec \in \Delta_{\rhovec}^{\ell}}  \langle \boldsymbol{\delta}_{t-1}, \xvec  \rangle + V_{\xvec_{t-1}}(\xvec) $ with $\boldsymbol{\delta}_{0} = \boldsymbol{0}$ and $\xvec_{0} \in \Delta_{\rhovec}^{\ell}$.
Then, the KL divergence $V_{\xvec}(\xvec^{'})$ satisfies
\begin{equation}
    \sum_{t =1}^{T} \langle \boldsymbol{\delta}_{t}, \xvec_{t} - \xvec_{T+1} \rangle \leq V_{\xvec_{0}}(\xvec_{T+1}) + \sum_{t =1}^{T} ||-\boldsymbol{\delta}_{t}||_{\xvec_{t}}^2. 
\end{equation}
\label{lem:Lem13Carmon}
\end{restatable}
Note that the results of Lemma \ref{lem:Lem12Carmon} (and Lemma \ref{lem:Lem13Carmon}) holds for any arbitrary sequence $\{\boldsymbol{\gamma}_{t}\}_{t=1}^{T}$ (and, $\{\boldsymbol{\delta}_{t}\}_{t=1}^{T}$). We are now ready to prove the Theorem. 
\begin{proof}[Proof of Theorem \ref{thm:SMDGapTheorem}]
First note that by the choice of $\eta^{\xvec}$ we have $||\eta^{\xvec}\tilde{\gvec}_{t}^{\xvec}||_{\infty} \leq 1/2$. Hence, we invoke Lemma \ref{lem:Lem13Carmon} with $ \boldsymbol{\delta}_{t} =  \eta^{\xvec} \tilde{\gvec}_{t}^{\xvec}$ and have  
\begin{equation}
    \sum_{t =1}^{T} \langle \eta^{\xvec} \tilde{\gvec}_{t}^{\xvec}, \xvec_{t} - \xvec_{T+1} \rangle \leq V_{\xvec_0}(\xvec_{T+1}) + \eta^{\xvec 2} \sum_{t=1}^{T} ||\tilde{\gvec}_{t}^{\xvec}||_{\xvec_t}^2.
    \end{equation}
    Similarly, using Lemma \ref{lem:Lem12Carmon} with $\gammavec_{t} = \eta^{\lambdavec} \tilde{\gvec}_{t}^ {\lambdavec} $ we have 
    \begin{equation}
      \sum_{t =1}^{T} \langle \eta^{\lambdavec} \tilde{\gvec}_{t}^{\lambdavec}, \lambdavec_{t} - \lambdavec_{T+1} \rangle \leq V_{\lambdavec_0}(\lambdavec_{T+1}) + \frac{\eta^{\lambdavec 2}}{2} \sum_{t=1}^{T} ||\tilde{\gvec}_{t}^{\lambdavec}||_2^2. \end{equation}
Let $\hat{\gvec}_{t}^{\xvec} := \gvec_t^{\xvec} - \tilde{\gvec}_{t}^{\xvec} $ and  $\hat{\gvec}_{t}^{\lambdavec} := \gvec_t^{\lambdavec} - \tilde{\gvec}_{t}^{\lambdavec} $  and define the sequence 
\begin{align*}
\hat{\xvec}_{1} = \xvec_{1}     , \hspace{20pt} & \hat{\xvec}_{t+1} = \arg\min_{\xvec \in \Delta_{\rhovec}^{\ell}} \langle \eta^{\xvec} \hat{\gvec}_{t}^{\xvec}, \xvec \rangle + V_{\hat{\xvec}_{t}}(\xvec) \\ 
\hat{\lambdavec}_{1} = \lambdavec_{1}     , \hspace{20pt} & \hat{\lambdavec}_{t+1} = \arg\min_{\lambdavec \in \mathbb{B}_{2M}^{\ell}} \langle \eta^{\lambdavec} \hat{\gvec}_{t}^{\lambdavec}, \lambdavec \rangle + V_{\hat{\lambdavec}_{t}}(\lambdavec) 
\end{align*}
Notice that,
\begin{align*}
||\eta^{\xvec}\hat{\gvec}_{t}^{\xvec}||_{\infty} & \underset{(i)}{\leq} ||\eta^{\xvec}\tilde{\gvec}_{t}^{\xvec}||_{\infty} + ||\eta^{\xvec}\gvec_{t}^{\xvec}||_{\infty}  \underset{(ii)}{\leq} ||\eta^{\xvec}\tilde{\gvec}_{t}^{\xvec}||_{\infty} + ||\eta^{\xvec}\mathbb{E}\tilde{\gvec}_{t}^{\xvec}||_{\infty}  \underset{(iii)}{\leq} ||\eta^{\xvec}\tilde{\gvec}_{t}^{\xvec}||_{\infty} + \mathbb{E}||\eta^{\xvec}\tilde{\gvec}_{t}^{\xvec}||_{\infty} \leq 1. 
\end{align*}
$(i)$ follows from the definition of $\hat{\gvec}_{t}^{\xvec}$, $(ii)$ follows from fact that $\tilde{\gvec}_t^{\xvec}$ is an unbiased estimator and finally $(iii)$ is due to Jensen's inequality. Invoke  Lemma \ref{lem:Lem13Carmon} by setting $\boldsymbol{\delta}_{t} = - \eta^{\xvec} \hat{\gvec}_{t}^{\xvec}$ we obtain, 
\begin{align}
    \sum_{t =1}^{T} \langle \eta^{\xvec} \hat{\gvec}_{t}^{\xvec}, \hat{\xvec}_{t} - \xvec_{T+1} \rangle \leq V_{\hat{\xvec}_0}(\xvec_{T+1}) + \eta^{\xvec 2} \sum_{t=1}^{T} ||\hat{\gvec}_{t}^{\xvec}||_{\hat{\xvec}_t}^2 \\
    \intertext{Similarly, for $\lambdavec$-space, using Lemma \ref{lem:Lem12Carmon} with $\gammavec_{t} = \eta^{\lambdavec} \hat{\gvec}_{t}^{\lambdavec}$, we have }
        \sum_{t =1}^{T} \langle \eta^{\lambdavec} \hat{\gvec}_{t}^{\lambdavec}, \hat{\lambdavec}_{t} - \lambdavec_{T+1} \rangle \leq V_{\hat{\lambdavec}_0}(\lambdavec_{T+1}) + \frac{\eta^{\lambdavec 2} }{2} \sum_{t=1}^{T} ||\hat{\gvec}_{t}^{\lambdavec}||_{2}^2 \end{align}
Since, $\gvec_t^{\xvec} = \tilde{\gvec}_t^{\xvec} + \hat{\gvec}_t^{\xvec}$ and  $ \gvec_t^{\lambdavec} = \tilde{\gvec}_t^{\lambdavec} + \hat{\gvec}_t^{\lambdavec}$  we can write, 
\begin{align*}
 & \sum_{t=1}^{T} [ \langle \gvec_t^{\xvec}, \xvec_t - \xvec_{T+1} \rangle + \langle \gvec_t^{\lambdavec}, \lambdavec_t - \lambdavec_{T+1} \rangle ] \\
=&   \sum_{t=1}^{T}[ \langle \tilde{\gvec}_t^{\xvec}, \xvec_t   - \xvec_{T+1} \rangle + \langle \tilde{\gvec}_t^{\lambdavec}, \lambdavec_t - \lambdavec_{T+1} \rangle  ]  + \sum_{t=1}^{T}[\langle \hat{\gvec}_t^{\xvec}, \xvec_t - \xvec_{T+1} \rangle + \langle \hat{\gvec}_t^{\lambdavec}, \lambdavec_t - \lambdavec_{T+1} \rangle ] \\ 
\leq&   \frac{1}{\eta^{\xvec}} V_{\xvec_{0}}(\xvec_{T+1})  + \frac{1}{\eta^{\xvec}} V_{\hat{\xvec}_{0}}(\hat{\xvec}_{T+1}) + \sum_{t=1}^{T} \eta^{\xvec} \Big[  ||\tilde{\gvec}_t^{\xvec}||_{\xvec_t}^{2} + ||\hat{\gvec}_t^{\xvec}||_{\hat{\xvec}_t}^{2}  \Big] 
+  \frac{1}{\eta^{\xvec}}
\sum_{t=1}^{T} \langle \hat{\gvec}_t^{\xvec}, \xvec_t - \hat{\xvec}_t  \rangle 
\\
+ & \frac{1}{\eta^{\lambdavec}} V_{\lambdavec_{0}}(\lambdavec_{T+1}) + \frac{1}{\eta^{\lambdavec}} V_{\hat{\lambdavec}_{0}}(\hat{\lambdavec}_{T+1}) + \sum_{t=1}^{T} \frac{\eta^{\lambdavec}}{2} \Big[ ||\tilde{\gvec}_t^{\lambdavec}||_{2}^{2} + ||\hat{\gvec}_t^{\lambdavec}||_{2}^{2} \Big] 
+ \frac{1}{\eta^{\lambdavec}}\sum_{t=1}^{T} \langle \hat{\gvec}_t^{\lambdavec}, \lambdavec_t - \hat{\lambdavec}_t \rangle 
\end{align*}
Note that $\mathbb{E} [ ||\hat{\gvec}^{\xvec}||_{\hat{\xvec}_t}^2] \leq \mathbb{E} [ ||\tilde{\gvec}^{\xvec}||_{\hat{\xvec}_t}^2] \leq v^{\xvec}$ because $\mathbb{E}[(X - \mathbb{E}[X])^2]  \leq \mathbb{E}[X]^2$   and property $(iii)$ of the bounded-estimator $\tilde{\gvec}^{\xvec}$ respectively.
First, taking the supremum over $(\xvec, \lambdavec )$ and then taking expectation we obtain, 
\begin{align}
    &\frac{1}{T} \mathbb{E} \left[\sup_{\xvec, \lambdavec }  \sum_{t=1}^{T} \langle \gvec_t^{\xvec}, \xvec_t - \xvec \rangle + \langle \gvec_t^{\lambdavec}, \lambdavec_t - \lambdavec \rangle \right] \nonumber \\ \underset{(i)}{ \leq} &\sup_{\xvec} \frac{2}{\eta^{\xvec} T } V_{\xvec_{0}}(\xvec) + 2 \eta^{\xvec} v^{\xvec} + \sup_{\lambdavec} \frac{2}{\eta^{\lambdavec} T } V_{\lambdavec_{0}}(\lambdavec)  \nonumber  + \eta^{\lambdavec} v^{\lambdavec}   \nonumber \\
    \underset{(ii)}{ \leq} &\frac{2\log(\ell)}{\eta^{\xvec}T} +  2 \eta^{\xvec} v^{\xvec} + \frac{4\ell M^2}{\eta^{\lambdavec}T} + \eta^{\lambdavec} v^{\lambdavec}   \underset{(iii)}{ \leq} \varepsilon \label{eqn:epsilon upper bound}
\end{align}
Here, $(i)$ follows from the fact that $\mathbb{E}[\langle \hat{\gvec}_{t}^{\xvec}, \xvec_{t} - \hat{\xvec}_{t} \rangle] = \mathbb{E}[\langle \hat{\gvec}_{t}^{\lambdavec}, \lambdavec_{t} - \hat{\lambdavec}_{t} \rangle]  = 0 $. This is true because $\tilde{\gvec}_{t}^{\xvec}$ and $ \tilde{\gvec}_{t}^{\lambdavec} $  are unbiased estimators. Next, $(ii)$ follows from the fact that the  KL divergence over $\Delta^{\ell}$ (and also over $\Delta_{\rhovec}^{\ell}$) is upper bounded by $\log(\ell)$ i.e. $V_{\xvec_{0}}(\xvec) \leq \log(\ell)$ for any $\xvec_0, \xvec \in \Delta^{\ell}$ and  $V_{\lambdavec_{0}}(\lambdavec) = \frac{1}{2} || \lambdavec_0 - \lambdavec||_2^2 \leq 2 \ell M^2 $. Finally, $(iii)$ follows directly from the choice of the parameters.Finally use Lemma \ref{lem:gapLemma} and Equation \ref{eqn:epsilon upper bound} to get the desired upper bound on the expected duality gap i.e. $\mathbb{E} \big [ \Gap((\xvec^{\varepsilon}, \lambdavec^{\varepsilon} ) ) \big ] \leq \varepsilon$.
\end{proof}

We now present the proof of supporting Lemmas. 
\begin{proof}[Proof of Lemma \ref{lem:gapLemma}]
For notational brevity denote $ \mathcal{U} := \Delta_{\rhovec}^{\ell} \times \mathbb{B}_{2M}^{n}  $ and $\uvec \in \mathcal{U}$ such that $\uvec_{x} = \xvec$ and $ \uvec_{\lambda} = \lambdavec$. Similarly we denote $\uvec^{\varepsilon} = (\xvec^{\varepsilon}, \lambdavec^{\epsilon})$ and $\uvec_{t} = (\xvec_{t}, \lambdavec_{t})$. First, define a function $$\gap(\uvec;\vvec):= f(\vvec_{x},\uvec_{\lambda}) - f(\uvec_{x}, \vvec_{\lambda}).$$ Note that $\gap(\uvec;\uvec) = 0$ for any $\uvec \in \mathcal{U}$ and $\Gap(\uvec) = \sup_{\vvec \in \mathcal{U}} \gap(\uvec; \vvec)$. Further, note that, for every $ \uvec \in \mathcal{U}$ the function  $\gap(\uvec;\vvec)$ is concave in $\vvec$, hence we have, 
\begin{equation}
 \gap(\uvec; \vvec) \leq \gap(\uvec; \uvec) + \langle \nabla_{\vvec}\gap(\uvec;\uvec) , \vvec - \uvec \rangle = \langle \gvec^{\uvec} , \uvec - \vvec \rangle  
 \label{eqn:concavity}
 \end{equation}
Note here that $\gvec^{\uvec} = - \nabla_{\vvec} \gap(\uvec; \uvec)$. Next, note that $\gap(\uvec; \vvec)$ is convex in $\uvec$ for any $\vvec \in \mathcal{U}$ hence we have, 
\begin{align}
   \gap(\uvec^{\varepsilon};\uvec)  & \leq \frac{1}{T} \sum_{t=1}^{T} \gap(\uvec_{t}; \uvec) \tag{By convexity of $\gap(.;.)$ w.r.t. the first argument} \nonumber \\
   & \leq \frac{1}{T} \sum_{t=1}^{T} \langle \gvec_t^{\uvec} , \uvec_t - \uvec \rangle \tag{By Eqn. \ref{eqn:concavity}} \nonumber \\ 
   & = \frac{1}{T} \sum_{t=1}^{T} \left[ \langle \gvec_t^{\xvec} , \xvec_t - \xvec \rangle + \langle \gvec_t^{\lambdavec} , \lambdavec_t - \lambdavec \rangle \right] \label{eqn:gapeqnTwo}
\end{align}
Furthermore, the space $\mathcal{U}$ is closed and the function $\gap(.;.)$ is continuous in both the arguments we have that the supremum is attained. Since the Equation \ref{eqn:gapeqnTwo} holds for every $\uvec \in \mathcal{U}$ it also holds for the supremum. This completes the proof.
\end{proof}

\begin{proof}[Proof of Lemma \ref{lem:Lem12Carmon}]

We begin with a simple observation for a distance generating function. 
\begin{restatable}{observation}{ObsOne}
For any $\xvec, \yvec, \zvec \in \mathcal{X}$ we have, $    - \langle \nabla V_{\xvec}(\yvec), \yvec - \zvec \rangle = V_{\xvec}(\zvec) - V_{\yvec}(\zvec) - V_{\xvec}(\yvec)$. 
\label{obs:one}
\end{restatable}
The proof of the observation follows directly from the definition of $V$. In what follows, we prove  lemmas used in establishing the proof of Theorem \ref{thm:SMDGapTheorem}. Let $\lambdavec_{0} \in \mathcal{X}$ and $\gammavec_{0} = \boldsymbol{0}$.    

From  the first order optimality condition  and convexity of $\mathcal{X}$ we have,
\begin{equation}
\langle \gammavec_{t-1} + \nabla V_{\lambdavec_{t-1}}(\lambdavec_{t}), \lambdavec_t - \lambdavec_{T+1}  \rangle \leq 0   
\label{eqn:firstOrder}
\end{equation}

Use Equation  \ref{eqn:firstOrder} and Observation~\ref{obs:one} with $\lambdavec := \lambdavec_{T+1}$  to get
\begin{align}
    \sum_{t=1}^{T} \langle \gammavec_{t-1}, \lambdavec_{t} - \lambdavec \rangle \nonumber &\leq - \sum_{t=1}^{T}\langle \nabla V_{\lambdavec_{t-1}}(\lambdavec_t), \lambdavec_t - \lambdavec) \rangle = \sum_{t=1}^{T} V_{\lambdavec_{t-1}}(\lambdavec) - V_{\lambdavec_{t}}(\lambdavec) - V_{\lambdavec_{t-1}}(\lambdavec_{t}) \\    
    &= V_{\lambdavec_{0}}(\lambdavec) + \sum_{t=1}^{T-1} V_{\lambdavec_{t}}(\lambdavec)  \nonumber  - \sum_{t=1}^{T}V_{\lambdavec_{t}}(\lambdavec) - \sum_{t=0}^{T-1} V_{\lambdavec_{t}}(\lambdavec_{t+1}) = V_{\lambdavec_{0}}(\lambdavec) - \sum_{t=0}^{T} V_{\lambdavec_{t}}(\lambdavec_{t+1})   
\end{align}
Now simplify the LHS of the above equation as follows. 
\begin{align*}
    \sum_{t=1}^{T} \langle \gammavec_{t-1}, \lambdavec_{t} - \lambdavec \rangle    = \sum_{t=1}^{T-1} \langle \gammavec_{t}, \lambdavec_{t} - \lambdavec \rangle - \sum_{t=1}^{T-1} \langle \gammavec_{t}, \lambdavec_{t} - \lambdavec_{t+1} \rangle   = \sum_{t=1}^{T} \langle \gammavec_{t}, \lambdavec_{t} - \lambdavec \rangle - \sum_{t=1}^{T} \langle \gammavec_{t}, \lambdavec_{t} - \lambdavec_{t+1} \rangle 
\end{align*}
    
The first equality follows from the fact that $\gammavec_{0} =\boldsymbol{0}$ and the second follows from $\lambdavec = \lambdavec_{T+1} $. Thus, we have 
\begin{align}
\sum_{t=1}^{T} \langle \gammavec_t, \lambdavec_{t} - \lambdavec \rangle  \leq & V_{\lambdavec_{0}}(\lambdavec) +  \sum_{t=0}^{T} \{ \langle \gammavec_{t}, \lambdavec_t - \lambdavec_{t+1}  \rangle - V_{\lambdavec_{t}}(\lambdavec_{t+1}) \}.   \label{eqn:sumInequality}
\end{align}
Furthermore, the following inequalities hold for every iteration $t$,  
$$
\langle \gamma_t, \lambdavec_{t} - \lambdavec_{t+1} \rangle \underset{(i)}{\leq} ||\gammavec_{t}||_{\star} ||\lambdavec_{t} - \lambdavec_{t+1}||   \underset{(ii)}{\leq} \frac{1}{2} ||\gammavec_t||_{\star}^2 + \frac{1}{2} ||\lambdavec_{t} - \lambdavec_{t+1} ||^2  \underset{(iii)}{\leq} \frac{1}{2} ||\gammavec_t||_{\star}^2 + V_{\lambdavec_{t}} (\lambdavec_{t+1}).    $$
In the above expression,  $(i)$ follows from H\"{o}lder's inequality, $(ii)$ follows from the AM-GM inequality and $(iii)$ follows from the strong convexity of the underlying distance-generating function. Using above inequality  with Equation \ref{eqn:sumInequality} completes the proof of the lemma. 
\end{proof}

\begin{proof}[Proof of Lemma \ref{lem:Lem13Carmon}]
Following the same steps to simplify the sum as in  Lemma \ref{lem:Lem12Carmon},  Equation \ref{eqn:sumInequality}   we obtain 
\begin{align}
    \sum_{t =1}^{T}  \langle \boldsymbol{\delta}_{t}, \xvec_{t} - \xvec_{T+1} \rangle  & \leq V_{\xvec_{0}}(\xvec_{T+1}) + \sum_{t =0 }^{T} \{\langle \boldsymbol{\delta}_t, \xvec_t - \xvec_{t+1} \rangle - V_{\xvec_t}(\xvec_{t+1}) \} \nonumber \\
    & = V_{\xvec_{0}}(\xvec_{T+1}) + \sum_{t =0}^{T} \{\langle - \boldsymbol{\delta}_t, \xvec_{t+1} - \xvec_{t} \rangle - V_{\xvec_t}(\xvec_{t+1}) \} \label{eqn: inner product inequality}
\end{align} 
We now provide the upper bound on the summation term in the RHS of Equation \ref{eqn: inner product inequality}. Recall that the  Fenchel  conjugate  $R^{\star}$ at point $\yvec$ is given  as 
\begin{equation*}
    \R^{\star}(\yvec) = \sup \{  \langle \yvec, \xvec \rangle - \R(\xvec) |\xvec \in \Delta_{\rhovec}^{\ell} \} 
\end{equation*}
This implies, 
\begin{equation}
    \R^{\star}(\yvec) \geq   \langle \yvec, \xvec \rangle - \R(\xvec)  \hspace{10pt} \text{for all } \xvec \in \Delta_{\rhovec}^{\ell} \hspace{10pt} \text{ and } \hspace{10pt} 
\R^{\star}(\nabla \R(\xvec)) = \langle \nabla \R(\xvec) , \xvec \rangle - \R(\xvec).
\label{eqn:maximizing}
\end{equation}
 Let $\yvec := \nabla R(\xvec) $. From maximizing argument we  have, $\nabla R^{\star}(\yvec) = \xvec$ hence we have   \begin{equation}
\xvec = \nabla \R^{\star}(\nabla \R(\xvec)).
\label{eqn:maximizing argument}
\end{equation}
Using definition of  $V_{\xvec}(\xvec')$,   Equation   \ref{eqn:maximizing} and Equation \ref{eqn:maximizing argument} we obtain, 
\begin{align}
  \langle  \boldsymbol{\delta}, \xvec'  - \xvec \rangle  - V_{\xvec}(\xvec')  & =  \langle\nabla \R(\xvec) + \boldsymbol{\delta}, \xvec' \rangle - \R(\xvec') - [ \langle \nabla \R(\xvec), \xvec \rangle - \R(\xvec) ] - \langle \xvec, \boldsymbol{\delta} \rangle \nonumber \\ & \leq \R^{\star}(\nabla \R(\xvec) + \boldsymbol{\delta}) - \R^{\star}(\nabla \R( \xvec)) -   \langle \nabla \R^{\star}(\nabla \R(\xvec)), \boldsymbol{\delta} \rangle \label{eqn:upprbndnorming} 
\end{align}
Note that  $R(.)$ is an entropic regularizer whose Fenchel dual  is given by $\R^{\star}(\xvec) = \log(\sum_{i=1}^{\ell} e^{\xvec_{i}})$ \citep[][Table 2.1]{shalev2011online}. Hence, 
\begin{align*}
\R^{\star}(\nabla \R(\xvec) + \boldsymbol{\delta}) - \R^{\star}(\nabla \R( \xvec))  &= \log \left(\frac{\sum_{i\in [\ell]} e^{(\nabla \R(\xvec) + \boldsymbol{\delta})_{i}}}{ \sum_{i \in [\ell]} e^{(\nabla \R(\xvec))_{i}}} \right) \\ 
& \leq \log \left( 1 + \frac{\sum_{i \in [\ell]} e^{(\nabla \R(\xvec))_{i}} (\delta_i + \delta_i^2)}{ \sum_{i \in [\ell]} e^{(\nabla \R(\xvec))_{i}} }\right) \tag{$e^{a} \leq 1+ a + a^2 $ for $a \leq 1.79$} \\ 
& = \log \left( 1+ \langle \nabla \R^{\star} (\nabla \R(\xvec)), \boldsymbol{\delta} + \boldsymbol{\delta}^{2} \rangle \right) \tag{$\boldsymbol{\delta}^2 $ is a vector with $i^{th}$ coordinate $\delta_{i}^2$} \\ 
&\leq \langle \nabla \R^{\star} (\nabla \R(\xvec)),  \boldsymbol{\delta} \rangle + \langle \nabla \R^{\star} (\nabla \R(\xvec)),  \boldsymbol{\delta}^{2} \rangle  \tag{$\log(1+a) \leq a$}
\end{align*}
Thus we have, $\R^{\star}(\nabla \R(\xvec) + \boldsymbol{\delta}) - \R^{\star}(\nabla \R( \xvec)) -   \langle \nabla \R^{\star}(\nabla \R(\xvec)), \boldsymbol{\delta} \rangle \leq  \langle \nabla \R^{\star} (\nabla \R(\xvec)),  \boldsymbol{\delta}^{2} \rangle = ||\boldsymbol{\delta}||_{\nabla \R^{\star}(\nabla \R(\xvec))}^2 = ||\boldsymbol{\delta}||_{\xvec}^2$. The last equality follows from Equation \ref{eqn:maximizing argument}.
        Using this in Equation \ref{eqn:upprbndnorming} we have 
        \begin{equation}
\langle \boldsymbol{\delta}, \xvec' - \xvec \rangle   - V_{\xvec}(\xvec') \leq  ||\boldsymbol{\delta}||_{\xvec}^2. \label{eqn:xvec local inequality}
\end{equation}
Finally, we complete the proof of the lemma by using the above result (Equation \ref{eqn:xvec local inequality}) in Equation \ref{eqn: inner product inequality} i.e. 
\begin{equation*}
\sum_{t =1}^{T}  \langle \boldsymbol{\delta}_{t}, \xvec_{t} - \xvec_{T+1} \rangle     \leq V_{\xvec_{0}}(\xvec_{T+1}) + \sum_{t =1}^{T} ||-\boldsymbol{\delta}_{t}||_{\xvec_{t}}^2 
\end{equation*}
\end{proof}

\section{Proof of Theorem \ref{thm: optimality and fairness}}\label{secappendix: proof of fairness and optimality}
\Optimalityfairnessthm*
In the proof of this theorem, for convenience we denote $\pi^\varepsilon$ as $\pi$, 
and $\boldsymbol{\nu}^{\pi^{\varepsilon}}$ equal to the stationary distribution corresponding to policy $\pi^{\varepsilon}$ as $\boldsymbol{\nu}$. Further let $\boldsymbol{u} \in \mathbb{R}^n$ be such that $u_s = \sum_{a\in [m]} x^{\varepsilon}_{s,a}$, and note that $x^{\varepsilon}_{s,a} = \pi_{s,a} u_s$ implying that $\boldsymbol{x}^{\varepsilon} = \Pi\cdot \boldsymbol{u}$, where $\Pi \in \mathbb{R}^{\ell \times n}$ is the policy matrix corresponding to $\pi$. Also, let $\boldsymbol{x} \in \mathbb{R}^{\ell}$ be such that $\boldsymbol{x} = \Pi \mathbf{\nu}$. Finally, note that $\mathbf{\Gamma}^{\pi} = \Pi^T\mathbf{\Gamma}$ is the probability transition matrix of the Markov chain induced by the policy $\pi$. Hence, 
\begin{equation}
(\mathbf{I}-\mathbf{\Gamma}^\pi)^T\boldsymbol{\nu}^\pi = \boldsymbol{0}~. 
\end{equation}
Finally, let $\boldsymbol{x}^*, \boldsymbol{\lambda}^*$ be the optimal solution to the problem in Equation \ref{equation: the final min maz optimization}. We first prove a few lemmas which would be used to prove the fairness and optimality.

\begin{lemma}\label{lemma: using epsilon gap for optimality}
$\mathbb{E}\left[ (\boldsymbol{x}^{\varepsilon})^T \left( (\biggamma - \Ihat) \boldsymbol{\lambda}^* - \boldsymbol{r} \right) + v^* \right] \leq \varepsilon$
\end{lemma}
\begin{proof}
From Theorem \ref{thm:SMDGapTheorem}, we have
$$\mathbb{E}[f(\boldsymbol{x}^*, \boldsymbol{\lambda}^{\varepsilon}) - f(\boldsymbol{x}^{\varepsilon}, \boldsymbol{\lambda}^{*})] \leq \varepsilon ~.$$
Substituting appropriate value of $f$ at the respective points and observing that $(\boldsymbol{x}^*)^T\boldsymbol{r} = v^*$ the lemma follows.
\end{proof}

\begin{lemma}\label{lemma: theorem 3 lemma 1}
$\mathbb{E}[\max_{\lambda \in \mathbb{B}_{2M}^n} \boldsymbol{u}^T(\I - \biggamma^{\pi})(\boldsymbol{\lambda}^{*} -\boldsymbol{\lambda})] \leq \varepsilon$
\end{lemma}
\begin{proof}
From Theorem \ref{thm:SMDGapTheorem}
\begin{equation}\label{equation: proof of theorem 3 eqn 1}
   f(\boldsymbol{x}^*, \boldsymbol{\lambda}^{\varepsilon}) -  \min_{\boldsymbol{\lambda} \in \mathbb{B}_{2M}^n} f(\boldsymbol{x}^{\varepsilon}, \boldsymbol{\lambda})  \leq \varepsilon
\end{equation}
Further note that since $(\boldsymbol{x}^{*})^T(\Ihat - \biggamma) = \boldsymbol{0}$, we have 
\begin{equation}\label{equation: proof of theorem 3 eqn 2}
   f(\boldsymbol{x}^{\varepsilon}, \boldsymbol{\lambda}^{*}) \leq   f(\boldsymbol{x}^*, \boldsymbol{\lambda}^{*}) = f(\boldsymbol{x}^{*}, \boldsymbol{\lambda}^{\varepsilon})  ~.
\end{equation}
Using Equations \ref{equation: proof of theorem 3 eqn 1} and \ref{equation: proof of theorem 3 eqn 2},
$$\mathbb{E}[\max_{\lambda \in \mathbb{B}_{2M}^n} \boldsymbol{u}^T(\I - \biggamma^{\pi})(\boldsymbol{\lambda}^{*} -\boldsymbol{\lambda})] \leq \varepsilon~.$$
\end{proof}

\begin{lemma}\label{lemma: bounding the l1 norm}
$\mathbb{E}||(\boldsymbol{u}- \boldsymbol{\nu}^\pi)^T(\I -\biggamma^\pi + \boldsymbol{1}(\boldsymbol{\nu}^\pi)^T)||_{1} \leq \frac{\varepsilon}{M}$~.
\end{lemma}
\begin{proof}
First we show that $\mathbb{E}||\boldsymbol{u}^T(\I -\biggamma^\pi)||_{1} \leq \frac{\varepsilon}{M}$:
\begin{align}
    2M \mathbb{E}||\boldsymbol{u}^T(\I -\biggamma^\pi)||_{1} &= \mathbb{E}\left[\max_{\boldsymbol{\lambda} \in \mathbb{B}_{2M}^n} \boldsymbol{u}^T(\I - \biggamma^\pi)(-\boldsymbol{\lambda}) \right]\\
    &=  \mathbb{E}\left[\max_{\boldsymbol{\lambda} \in \mathbb{B}_{2M}^n}  \boldsymbol{u}^T(\I - \biggamma^{\pi})(\boldsymbol{\lambda}^{*} -\boldsymbol{\lambda}) - \boldsymbol{u}^T(\I -\biggamma^\pi)\boldsymbol{\lambda}^* \right]\\
    & \leq \varepsilon + ||\boldsymbol{\lambda}^{*}||_{\infty}\mathbb{E}||\boldsymbol{u}^T(\I -\biggamma^\pi)||_{1} \leq \varepsilon + M\mathbb{E}||\boldsymbol{u}^T(\I -\biggamma^\pi)||_{1}~.
\end{align}
The last but one inequality follows from Lemma \ref{lemma: theorem 3 lemma 1} and $||\boldsymbol{\lambda}^*||_{\infty}\leq M$. Here, we have made use of the fact that the box is over the range $2M$ whereas $||\boldsymbol{\lambda}^*||_{\infty}\leq M$. In particular, this is the place where enlarging the box helps (as in \cite{jin2020efficiently}).
Finally, to prove the lemma observe that $(\boldsymbol{\nu}^\pi)^T(\I-\biggamma^\pi) = \boldsymbol{0}$ and $|| (\boldsymbol{u}- \boldsymbol{\nu}^\pi)^T(\boldsymbol{1}(\boldsymbol{\nu}^\pi)^T)||_1 = 0$, and hence
$$\mathbb{E}||(\boldsymbol{u}- \boldsymbol{\nu}^\pi)^T(\I -\biggamma^\pi + \boldsymbol{1}(\boldsymbol{\nu}^\pi)^T)||_{1} = \mathbb{E}||(\boldsymbol{u}- \boldsymbol{\nu}^\pi)^T(\I -\biggamma^\pi)||_{1} \leq \frac{\varepsilon}{M}~.$$
\end{proof}

\begin{lemma}\label{lemma: optimality term 2}
Let $\boldsymbol{r}^{\pi} = \Pi^T\boldsymbol{r}$. Then $\mathbb{E}\left[(\boldsymbol{u} - \boldsymbol{\nu}^\pi)^T \boldsymbol{r}^\pi\right] \leq \varepsilon$ .
\end{lemma}
\begin{proof}
The proof of the lemma is completed using the following sequence of equations.
\begin{align}
    \mathbb{E}\left[ (\boldsymbol{u}- \boldsymbol{\nu}^{\pi})^T\boldsymbol{r}^\pi \right] 
    &= \mathbb{E}\left[ (\boldsymbol{u}- \boldsymbol{\nu}^{\pi})^T (\I - \biggamma^\pi + \boldsymbol{1}(\boldsymbol{\nu}^\pi)^T) (\I - \biggamma^\pi + \boldsymbol{1}(\boldsymbol{\nu}^\pi)^T)^{-1}
    \boldsymbol{r}^\pi \right] \\
    &\leq \mathbb{E}||(\boldsymbol{u}- \boldsymbol{\nu}^{\pi})^T (\I - \biggamma^\pi + \boldsymbol{1}(\boldsymbol{\nu}^\pi)^T) ||_{1} ||(\I - \biggamma^\pi + \boldsymbol{1}(\boldsymbol{\nu}^\pi)^T)^{-1}
    \boldsymbol{r}^\pi||_{\infty} \\
    &\leq \mathbb{E}||(\boldsymbol{u}- \boldsymbol{\nu}^{\pi})^T (\I - \biggamma^\pi + \boldsymbol{1}(\boldsymbol{\nu}^\pi)^T) ||_{1} ||(\I - \biggamma^\pi + \boldsymbol{1}(\boldsymbol{\nu}^\pi)^T)^{-1}||_{\infty} ||\boldsymbol{r}^\pi||_{\infty} \\
    &\leq \frac{\varepsilon}{M}\cdot 2t_{\text{mix}} \leq \varepsilon~.
\end{align}
The first inequality in the last line follows from Lemmas   \ref{lemma: from JIN SIDFORD} and \ref{lemma: bounding the l1 norm}. The last inequality follows by observing that $M \geq 2t_{\text{mix}}$.
\end{proof}

\noindent \textbf{Proof of Fairness}: Recall that $\mathbf{D}_{\rhovec}$ is the $n\times n$ diagonal matrix with its $s$-th entry being $\frac{1}{\rho_s}$. It is easy to see that $\mathsf{C}\Pi = \mathbf{D}_{\boldsymbol{\rho}}$, Now we show that $\mathbf{D}_{\boldsymbol{\rho}} \boldsymbol{\nu}^\pi = \mathbf{C}\boldsymbol{x}  \geq \boldsymbol{1}- \varepsilon$. First we note that as $\xvec_t \in \Delta_{\rhovec}^{\ell}$ for all $t\in [T]$, $\boldsymbol{x}^{\varepsilon} \in \Delta^{\ell}_{\rhovec}$. Hence, we have $\mathbf{D}_{\rhovec}\boldsymbol{x}^\varepsilon \geq \boldsymbol{1}$. The policy $\pi$ is $\varepsilon$-fair follows from sequence of equations below:

\begin{align*}
    \mathbf{C}\boldsymbol{x} &= \mathbf{C}\boldsymbol{x}^\varepsilon + \mathbf{C}(\boldsymbol{x}-\boldsymbol{x}^\varepsilon) =  \mathbf{C}\boldsymbol{x}^\varepsilon + \mathbf{C}\Pi(\boldsymbol{\nu}^\pi - \boldsymbol{u}) \\
    &= \mathbf{C}\boldsymbol{x}^\varepsilon + \mathbf{D}_{\boldsymbol{\rho}} (\I - (\biggamma^\pi)^T + \boldsymbol{\nu}^\pi\boldsymbol{1}^T)^{-1} (\I - (\biggamma^\pi)^T + \boldsymbol{\nu}^\pi\boldsymbol{1}^T)(\boldsymbol{\nu}^\pi - \boldsymbol{u}) \\
    &\geq \boldsymbol{1} - ||\mathbf{D}_{\boldsymbol{\rho}}(\I - (\biggamma^\pi)^T + \boldsymbol{\nu}^\pi\boldsymbol{1}^T)^{-1} (\I - (\biggamma^\pi)^T + \boldsymbol{\nu}^\pi\boldsymbol{1}^T)(\boldsymbol{\nu}^\pi - \boldsymbol{u})||_{\infty}\cdot \boldsymbol{1}\\
    &\geq \boldsymbol{1} - \left(|| \mathbf{D}_{\boldsymbol{\rho}} (\I - (\biggamma^\pi)^T + \boldsymbol{\nu}^\pi\boldsymbol{1}^T)^{-1}||_{\infty} ||(\I - (\biggamma^\pi)^T + \boldsymbol{\nu}^\pi\boldsymbol{1}^T)(\boldsymbol{\nu}^\pi - \boldsymbol{u})||_{1}\right)\cdot \boldsymbol{1} \\
    &\geq \boldsymbol{1} - ((\max_s \frac{1}{\rho_s})\cdot 2t_{\text{mix}}\cdot \frac{\varepsilon}{M})\cdot \boldsymbol{1} \geq \boldsymbol{1} - \varepsilon\boldsymbol{1}~.
\end{align*}
The first inequality in the last line follows from Lemmas \ref{lemma: from JIN SIDFORD} and \ref{lemma: bounding the l1 norm}, and the last inequality follows by observing that $M \geq (\max_s \frac{1}{\rho_s})\cdot 2t_{\text{mix}}$.\\

\noindent \textbf{Proof of Optimality}: We have the following sequence of equations.
\begin{align*}
   v^{\pi} = \mathbb{E}\left[(\boldsymbol{\nu}^\pi)^T \boldsymbol{r}^\pi \right] &= \mathbb{E}\left[(\boldsymbol{\nu}^\pi)^T(\I - \biggamma^\pi)\boldsymbol{\lambda^*} +  (\boldsymbol{\nu}^\pi)^T \boldsymbol{r}^\pi \right] \\
    &= \mathbb{E}\left[(\boldsymbol{\nu}^\pi - \boldsymbol{u})^T((\I - \biggamma^\pi)\lambda^* + \boldsymbol{r}^\pi) +  \boldsymbol{u}^T((\I - \biggamma^\pi)\boldsymbol{\lambda^*} + \boldsymbol{r}^\pi) \right] \\
    &\geq \mathbb{E}\left[(\boldsymbol{\nu}^\pi - \boldsymbol{u})^T(\I - \biggamma^\pi)\boldsymbol{\lambda^*}\right] + \mathbb{E}\left[(\boldsymbol{\nu}^\pi - \boldsymbol{u})^T \boldsymbol{r}^\pi\right] + v^* - \varepsilon \\
    & \geq v^* - 3\varepsilon~.
\end{align*}
The last but one inequality follows from Lemma \ref{lemma: using epsilon gap for optimality}, and the last inequality follows from Lemmas \ref{lemma: bounding the l1 norm} and \ref{lemma: optimality term 2}. 

\noindent \textbf{Sample Complexity:} The sample complexity result follows easily by replacing the values of $\ell =nm$,  size of the  bounding box $M = 2t_{mix}(1+ d_{\rhovec})$ and  step sizes  $\eta^{\xvec} = \frac{\varepsilon }{8 \ell (24M^2 + 1)}$ and $\eta^{\lambdavec} = \frac{\varepsilon}{16} $ as specified in Line 3 of Algorithm \ref{algorithm: fair state visitation}. In particular, we have $ \frac{8 \log(\ell)}{\eta^{\xvec}\varepsilon} =   \frac{ 64 \ell (24M^2 + 1)\log(\ell)}{ \varepsilon^2} = 6144 \varepsilon^{-2} nm t_{mix}^2(1+ d_{\rhovec})^2   \log(nm) + 64 \varepsilon^{-2}nm \log(nm) $ and $ \frac{32M^2n}{\eta^{\lambdavec}\varepsilon} =  2048 \varepsilon^{-2}n m t_{mix}^2 (1+ d_{\rhovec})^2$. Hence,  Algorithm \ref{algorithm: fair state visitation} has  sample  complexity of  $O(nm\varepsilon^{-2} t_{mix}^2 (1+ d_{\rhovec})^2 \log(nm))$.

\fi

\end{document}